\documentclass{article}

\PassOptionsToPackage{numbers, compress}{natbib}

\usepackage[preprint]{neurips_2025}

\usepackage{amsmath}
\usepackage{amssymb}
\usepackage{mathtools}
\usepackage{amsthm}
\usepackage{graphicx}
\usepackage{subfigure}
\usepackage{enumitem}

\usepackage[utf8]{inputenc} 
\usepackage[T1]{fontenc}    
\usepackage{hyperref}       
\usepackage{url}            
\usepackage{booktabs}       
\usepackage{amsfonts}       
\usepackage{nicefrac}       
\usepackage{microtype}      
\usepackage{xcolor}         

\usepackage{multirow}
\usepackage{wrapfig}

\theoremstyle{definition}
\newtheorem{theorem}{Theorem}[section]
\newtheorem{proposition}[theorem]{Proposition}

\theoremstyle{definition}

\theoremstyle{remark}

\usepackage{graphicx}
\usepackage[normalem]{ulem}
\useunder{\uline}{\ul}{}

\hypersetup{hidelinks,
	colorlinks=true,
	allcolors=black,
	pdfstartview=Fit,
	breaklinks=true}

\title{Adaptive Tokenization: On the Hop-Overpriority Problem in Tokenized Graph Learning Models}

\author{
  Zhibiao Wang\textsuperscript{1},
  Yunlong Zhou\textsuperscript{1},
  Ziwei Zhang\textsuperscript{1},
  Mengmei Zhang\textsuperscript{2},
  Shirui Pan\textsuperscript{3},\\
  \textbf{Chunming Hu\textsuperscript{1},}
  \textbf{Xiao Wang\textsuperscript{1}\thanks{Corresponding authors.}\hspace{2mm}}\\
  \textsuperscript{1} Beihang University  \qquad \textsuperscript{2} China Telecom Bestpay  \qquad \textsuperscript{3} Griffith University\\
  \texttt{\{wzb2321, yunlonghere, zwzhang, hucm, xiao\_wang\}@buaa.edu.cn},\\
  \texttt{zhangmengmei@bestpay.com.cn},
  \texttt{s.pan@griffith.edu.au}
  \vspace{-0.5em}
}

\newcommand{\MethodName}{LGTL}
\begin{document}

\maketitle

\begin{abstract}
 Graph Transformers, leveraging the global attention to capture long-range dependencies in graph structures, have significantly advanced graph machine learning, but face prohibitive computational complexity. Tokenized Graph Learning Models (TGLMs) address this issue by converting graphs into ordered token lists for scalable processing. Besides, TGLMs also empower Large Language Models (LLMs) to handle text-attributed graphs more effectively, and thus are also employed in Graph LLMs.  
 However, existing TGLMs rely on hand-designed token lists and their adaptability to diverse graph learning scenarios remains unexplored. In this paper, we first conduct extensive empirical and theoretical preliminary studies for hand-designed token lists. Surprisingly, we identify an unexplored ``hop-overpriority problem'': the common predefined token lists overemphasize nearby nodes and overwhelm the ability of TGLMs to balance local and global signals. This phenomenon is especially harmful for heterophilic graphs.
 To address this problem, we propose the Learnable Graph Token List (\MethodName), a plug-and-play module to replace hand-designed token lists in TGLMs. Specifically, \MethodName~adaptively adjusts the weights across hops and prioritizes informative nodes within hops through a graph attention gate module and a selection module, respectively. In this way, contextually informative nodes can be adaptively emphasized for both homophilic and heterophilic graphs. Besides, we theoretically show that \MethodName~can address the hop-overpriority problem. Extensive experiments on  benchmarks validate the efficacy of \MethodName~across both Graph Transformers and Graph LLM backbones.
\end{abstract}

\section{Introduction}\label{sec:intro}
Graph data, as a powerful data structure for modeling relational information, is ubiquitous in real-world systems, ranging from social networks, citation networks, to molecular interaction networks~\citep{yang2017neural, huang2020skipgnn}. To enable effective learning for graph data, Graph Neural Networks (GNNs)~\citep{kipf2016semi, velivckovic2017graph, hamilton2017inductive, xu2018powerful, wu2019simplifyinggraphconvolutionalnetworks} have been proposed, which use message-passing to capture local structural signals and learn high-quality representations from graph data. However, GNNs face challenges such as over-smoothing and over-squashing with deeper layers~\citep{li2018deeper, nt2019revisiting, oono2019graph, topping2021understanding, deac2022expander}. To tackle these issues, Graph Transformers adopt the global attention mechanism of Transformers~\citep{vaswani2017attention} to model long-range dependencies. Despite their effectiveness, Graph Transformers typically need to attend every pair of nodes in the input graph~\citep{ying2021transformers, dwivedi2020generalization}, therefore incur high computational costs and limit their scalability for large graphs.


More recently, to address the scalability limitations of global-attention Graph Transformers, Tokenized Graph Learning Models (TGLMs) have emerged as a promising paradigm by converting graphs into node-centric token lists (e.g., sequences of aggregated neighborhoods or sampled nodes)~\citep{chen2022nagphormer, shirzad2023exphormer, fu2024vcr}. By reducing the input to token sequences with fixed lengths, TGLMs enable efficient attention computation while preserving the global signals via attention mechanisms. Besides accelerating conventional Graph Transformers, the tokenization approach aligns naturally with Large Language Models (LLMs), which  has recently drawn considerable attention in the graph machine learning community, particularly for text-attributed graphs~\cite{li2024survey,ren2024survey,zhang2023graph}. As a result, TGLMs have also been adopted in Graph LLMs to bridge graph structures with LLMs, enabling the powerful modeling and reasoning abilities of LLMs to more effectively handle graph tasks~\citep{chen2024llaga, tang2024graphgpt}.

Despite their initial successes, the effectiveness of TGLMs critically relies on the design of the graph token list fed into the model. Existing works have proposed diverse strategies for token lists. For instance, VCR-Graphormer~\citep{fu2024vcr}, a representative Graph Transformer, uses personalized PageRank (PPR)~\citep{gasteiger2022predictpropagategraphneural} to inject cluster-level context into token lists. LLaGA~\citep{chen2024llaga}, a representative LLM for graph modeling, employs two fixed-template token lists: one aggregates neighborhood information via average pooling to form central node tokens, and another recursively samples neighborhood nodes to construct token sequences. Although these fixed-template token lists have shown performance in certain scenarios, a critical question remains unexplored: 

\quad \textit{Do pre-defined token lists universally enhance TGLMs, or do they fail under certain scenarios?} 

Investigating this question is critical given the structural diversity of real-world graphs. For example, many practical graphs exhibit relational patterns where neighborhood nodes carry inconsistent signals. Pre-defined strategies with fixed templates could inadvertently prioritize uninformative neighbors. 

To answer this question, we conduct preliminary experiments for different token lists templates (please refer to Section~\ref{sec:problem} for detailed settings and results). Strikingly, we observe that pre-defined token lists cannot universally enhance performance, but rather severely deteriorate the performance, especially on heterophilic graphs. The results are consistent across various datasets. To further analyze this phenomenon, we conduct extensive theoretical analyses for the effects of pre-defined token lists, especially for their failure cases. Our results reveal a previous unnoticed \textbf{hop-overpriority problem}: pre-defined strategies explicitly amplify the attention weights of nearby nodes in the token list, overwhelming the model's ability to balance local and global signals. For example, on heterophilic graphs, where 1-hop neighbors are less informative due to low homophily, this over-prioritization of local nodes forces the model to rely on noisy signals, leading to suboptimal performance.


Inspired by our preliminary analyses and to tackle the hop-overpriority problem, we propose Learnable Graph Token List (\MethodName), a plug-and-play module designed to replace pre-defined graph token lists in TGLMs. Specifically, unlike fixed templates, LGTL adaptively assigns weights to nodes across different hops using a graph-attention gate module. This adaptive weighting allows \MethodName~to emphasize informative nodes contextually for both homophilic and heterophilic graphs. Furthermore, \MethodName~adopts a selection module to assign distinct weights to nodes within-hop, distinguishing the informativeness of individual neighbors beyond hop-level aggregation. We show that \MethodName~can be easily integrated into various TGLMs, including both Graph Transformers and Graph LLMs. Besides, we provide theoretical analyses to show that \MethodName~is effective in addressing the hop-overpriority problem. 
Extensive experiments on various benchmarks and diverse TGLM backbones validate that \MethodName~significantly improves the performance and effectively mitigates the hop-overpriority problem. We summarize our contributions as follows:
\begin{itemize}[leftmargin = 0.5cm]
    \item 
    We empirically and theoretically characterize the hop-overpriority problem, a critical yet unexplored problem in pre-defined token lists for tokenized graph learning models covering both Graph Transformers and Graph LLMs, which is especially important for heterophilic graphs.
    \item We propose \MethodName, a flexible tokenization method that adaptively adjusts hop weights and prioritizes informative nodes within hops. We theoretically prove that \MethodName~can address the hop-overpriority problem. 
    \item We conduct extensive experiments on both homophilic and heterophilic datasets with various Graph Transformers and Graph LLMs as the backbone. Experiments demonstrate the effectiveness, compatibility and broad applicability of our method.  
\end{itemize}

\section{Related Works}\label{sec:relatedwork}

\textbf{Graph Transformers and Tokenization}. Graph Transformers, inspired by the attention mechanism of standard Transformers~\citep{vaswani2017attention}, have advanced graph representation learning by capturing global structural dependencies~\citep{ying2021transformers, kreuzer2021rethinking, bo2023specformer, mialon2021graphit, wu2021representing, chen2022structure, kim2022pure, rampavsek2022recipe, ma2023graph}. Early works like GT~\citep{dwivedi2020generalization} treat nodes as tokens and use dense self-attention over all node pairs to model interactions. However, this global-attention design faces scalability challenges for large graphs, spurring the development of tokenization-based architectures. Subsequent graph transformers with tokenization, such as NAGphormer~\citep{chen2022nagphormer}, shift focus to \textit{node-specific token lists}, typically constructed via neighborhood aggregation. These models apply self-attention only within each node's token list, enabling scalable mini-batch training. Follow-up works further improve token lists by integrating richer local context to balance local focus with global awareness~\citep{shirzad2023exphormer, fu2024vcr, chen2024leveraging}. Despite these advances, existing tokenization strategies 
rely on \textit{pre-defined templates}, e.g., fixed neighbor sampling or aggregation, lacking adaptability to diverse graph structures and limiting their ability capture task-relevant signals.

\textbf{LLM for Graphs}. LLMs have recently been extensively studied for graph data, particularly Text-Attributed Graphs (TAGs). They can be broadly categorized into two paradigms: text-based Graph LLMs and token-based Graph LLMs~\cite{li2024survey,ren2024survey,zhang2023graph}. Text-based approaches query LLMs using textual representations of graphs. Early works design prompts to encode graph structure and node features into natural language (e.g., describing nodes with their text and neighbors)~\citep{chen2024exploring, huang2023can}. Subsequent efforts refine textual formats to improve LLM understanding, such as syntax trees~\citep{zhao2023graphtext}, random walks~\citep{tan2023walklm}, or code-like descriptions~\citep{wang2024instructgraph}. However, these methods often struggle with scalability due to the linear growth of text length with graph size. Token-based approaches address this issue by compressing graph structures and text features into token-level embeddings. Representative models like LLaGA~\citep{chen2024llaga}, GraphGPT~\citep{tang2024graphgpt}, and GraphTranslator~\citep{zhang2024graphtranslator} construct node-specific token lists to integrate graphs into token spaces of LLMs. 
However, existing token-based works rely on \textit{pre-defined token lists}, which fail to adaptively handle diverse graph structure 
and motivates our work to design adaptive token lists that align with graph-specific properties.

\section{Preliminaries}\label{sec:preliminaries}
In this section, we introduce the notations and preliminaries for tokenized graph learning models and two common templates.

\textbf{Notations}: we denote a graph as $\mathcal{G} = (\mathcal{V}, \mathcal{E})$, where $\mathcal{V}$ represents the set of $N$ nodes and $\mathcal{E}$ denotes the set of $M$ edges. 
Each node $u$ is associated with a $d$-dimensional feature $\mathbf{x}_u$, forming the node feature matrix $\mathbf{X} \in \mathbb{R}^{N \times d}$. For node $u$, its $k$-hop neighborhood $\mathcal{N}_u^k$ refers to nodes reachable from $u$ via exactly $k$ edges and $\mathcal{N}_u = \mathcal{N}_u^1$. A graph token list $\mathbf{T} = [\mathbf{T}_1, \mathbf{T}_2, \dots, \mathbf{T}_L] \in \mathbb{R}^{L\times d}$ is a sequence of $L$ graph tokens, where each token $\mathbf{T}_i$ is a weighted combination of node features.

\textbf{Tokenized Graph Learning Models (TGLMs)}: they are designed to process the input graph token list and learn graph representations for various downstream tasks. Models of transformer-based architecture (e.g., Graph Transformers or LLMs) allow the central node to attend to other nodes by global attention mechanism as follows:
\begin{equation}\label{eq:global_attn}
 \text{Attn}(\mathbf{T}) = \text{Softmax}\left(\frac{\mathbf{Q}\mathbf{K}^\top}{\sqrt{h}}\right)\mathbf{V}, \text{where }
\mathbf{Q}=\mathbf{TW}_Q,\mathbf{K}=\mathbf{TW}_K,\mathbf{V}=\mathbf{TW}_V,
\end{equation}
where $\mathbf{W}_Q, \mathbf{W}_K, \mathbf{W}_V \in \mathbb{R}^{h \times h}$ are trainable weight matrices. 
From Eq.~\eqref{eq:global_attn}, it is evident that the design of the token lists $\mathbf{T}$ is critical to TGLMs, as they determine the input signals for the transformer-based architecture.
Here, we introduce two classical templates for the graph token list.

\textbf{Hop-field Overview Template (HO)}. It aims to summarize the signal of multi-hop neighbors using aggregated hop embeddings. It employs parameter-free message passing on node features to compute hop-specific representations. For the central node $u$, given $\mathbf{h}_u^0 = \mathbf{x}_u$, the $k$-th graph token $\mathbf{h}_u^k$ is defined as $\mathbf{h}_u^k = \frac{1}{|\mathcal{N}_u|} \sum_{v \in \mathcal{N}_u} \mathbf{h}_{v}^{k-1}$, which recursively aggregates the signal into a single embedding.

\textbf{Neighborhood Detail Template (ND).} Given the central node $u$, ND constructs a fixed size of computational tree rooted at $u$. Denote the neighbor sample size for the $k$-th hop as $n_k$. Starting from the root node $u$, $n_1$ nodes are sampled from its 1-hop neighborhood $\mathbf{N}_u^1$ to form $\widetilde{\mathbf{N}}_u^1$, and each node in $\widetilde{\mathbf{N}}_u^1$ recursively samples $n_2$ nodes from their own 1-hop neighborhoods, and this process repeats for the remaining hops until the $k$-th hop.

These two templates are representative of prevalent token list construction strategies in existing graph learning methods, covering both Graph Transformers such as Gophormer~\citep{zhao2021gophormer}, NAGphormer~\cite{chen2022nagphormer}, VCR-Graphormer~\cite{fu2024vcr}, and Graph LLMs such as LLaGA~\cite{chen2024llaga} and GraphGPT~\cite{tang2024graphgpt}. 
\section{Hop-Overpriority Problem for Tokenized Graph Learning Models}\label{sec:problem}
In this section, we introduce empirical and theoretical results for the hop-overpriority problem for tokenized graph learning models. 
\subsection{Empirical Observations}
To empirically explore the impact of graph token lists for tokenized graph learning models, we conduct preliminary experiments using a representative Graph LLM, LLaGA~\citep{chen2024llaga}. 
Specifically, we use a a frozen LLM to ensure that performance differences stem solely from graph token lists. We evaluate the performance of LLaGA on homophilic and heterophilic graphs with three types of graph token lists: No Template (NONE), HO, and ND.
The details of the datasets are provided in~\ref{appendix:datasets_llaga}.

\begin{table}[h]
  \caption{The performance of different graph token lists using LLaGA. The value below the dataset indicates the edge homophily. Numbers in parentheses indicate comparing with the None method.
  }
  \label{empirical_observations}
  \centering
  \large
  \resizebox{1.0\linewidth}{!}{
  \begin{tabular}{cccccccc}
\toprule
\multirow{2}{*}{Templates} & Dataset & Cora  & PubMed  & Cornell  & Texas  & Wisconsin  & Actor  \\ 
& Homophily & 0.8138 & 0.8024 & 0.1360 & 0.1452 & 0.2199 & 0.5608 \\ 
\midrule
\multirow{2}{*}{None} & Micro-F1 & 84.13 & 94.88 & 64.67 & 87.10 & 72.92 & 76.15 \\ 
& Macro-F1 & 82.05 & 94.42 & 50.00 & 83.93 & 64.76 & 70.12 \\ \hline
\multirow{2}{*}{HO} & Micro-F1 & 89.22 \scalebox{0.75}{($\color{green}\uparrow{5.09}$)} & 95.03 \scalebox{0.75}{($\color{green}\uparrow{0.15}$)} & 42.67 \scalebox{0.75}{($\color{red}\downarrow{22.00}$)} & 61.29 \scalebox{0.75}{($\color{red}\downarrow{25.81}$)} & 49.58 \scalebox{0.75}{($\color{red}\downarrow{23.34}$)} & 77.05 \scalebox{0.75}{($\color{green}\uparrow{0.90}$)} \\ 
& Macro-F1 & 87.65 \scalebox{0.75}{($\color{green}\uparrow{5.60}$)} & 94.56 \scalebox{0.75}{($\color{green}\uparrow{0.14}$)} & 36.25 \scalebox{0.75}{($\color{red}\downarrow{13.75}$)} & 57.44 \scalebox{0.75}{($\color{red}\downarrow{26.49}$)} & 32.04 \scalebox{0.75}{($\color{red}\downarrow{32.72}$)} & 70.48 \scalebox{0.75}{($\color{green}\uparrow{0.36}$)} \\ \hline
\multirow{2}{*}{ND} & Micro-F1 & 88.86 \scalebox{0.75}{($\color{green}\uparrow{4.73}$)} & 95.03 \scalebox{0.75}{($\color{green}\uparrow{0.15}$)} & 46.67 \scalebox{0.75}{($\color{red}\downarrow{18.00}$)} & 74.19 \scalebox{0.75}{($\color{red}\downarrow{12.91}$)} & 50.83 \scalebox{0.75}{($\color{red}\downarrow{22.09}$)} & 77.34 \scalebox{0.75}{($\color{green}\uparrow{1.19}$)} \\ 
& Macro-F1 & 86.71 \scalebox{0.75}{($\color{green}\uparrow{4.21}$)} & 94.56 \scalebox{0.75}{($\color{green}\uparrow{0.14}$)} & 42.35 \scalebox{0.75}{($\color{red}\downarrow{7.65}$)} & 68.93 \scalebox{0.75}{($\color{red}\downarrow{15.00}$)} & 34.06 \scalebox{0.75}{($\color{red}\downarrow{30.70}$)} & 72.99 \scalebox{0.75}{($\color{green}\uparrow{2.87}$)} \\ \bottomrule
\end{tabular}
}
\end{table}

The results, as shown in Table~\ref{empirical_observations}, reveal that both ND and HO templates achieve better performance than None on homophilic graphs. However, on heterophilic graphs, e.g., Cornell, Texas, and Wisconsin with low edge homophily, the predefined templates lead to significant performance degradation. 
This indicates that the predefined templates may introduce task-irrelevant or even harmful features when processing heterophilic graphs. 
%
A plausible reason is that heterophilic graphs have sparse intraclass edges, and the predefined templates aggregate irrelevant or conflicting features into the graph token list.

\begin{figure*}[h!]
\vspace{-1.0em}
\centering
\subfigure[Cora] {
  \begin{minipage}[b]{.142\linewidth}
    \centering
    \includegraphics[scale=0.086]{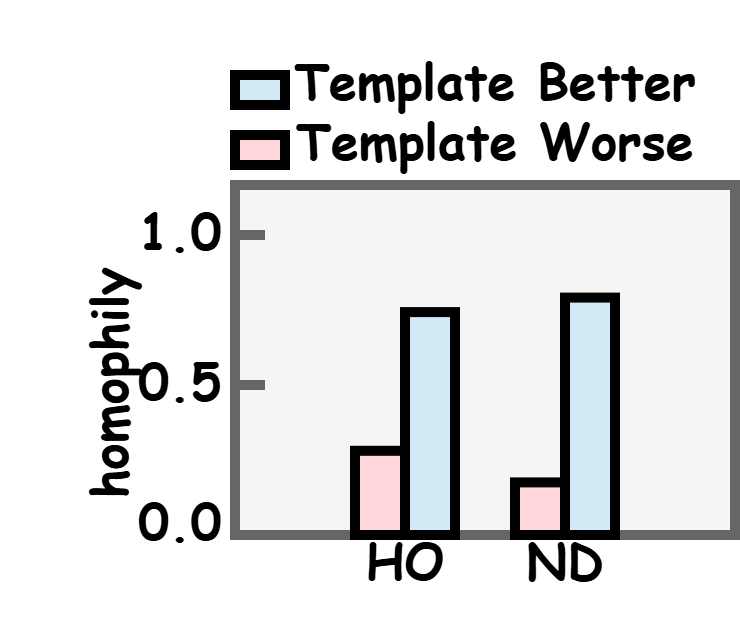}
  \end{minipage}
}
\subfigure[PubMed] {
  \begin{minipage}[b]{.142\linewidth}
    \centering
    \includegraphics[scale=0.086]{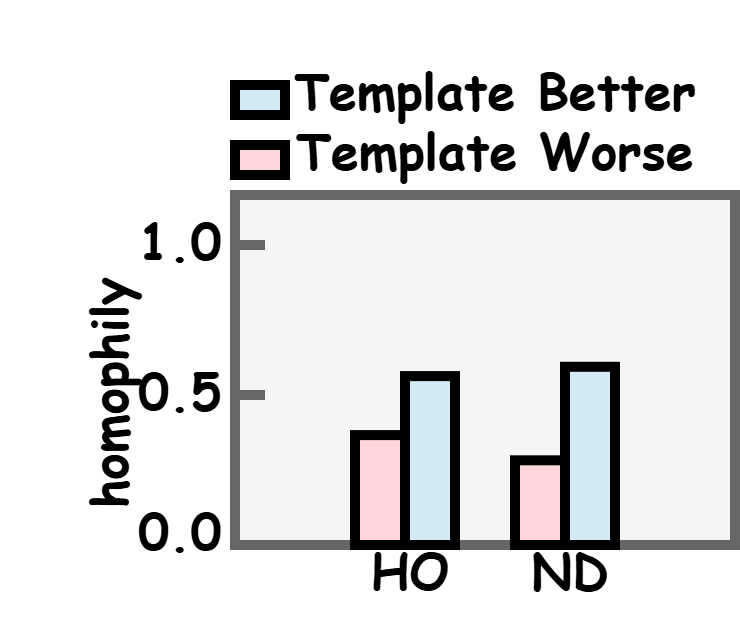}
  \end{minipage}
}
\subfigure[Actor] {
  \begin{minipage}[b]{.142\linewidth}
    \centering
    \includegraphics[scale=0.086]{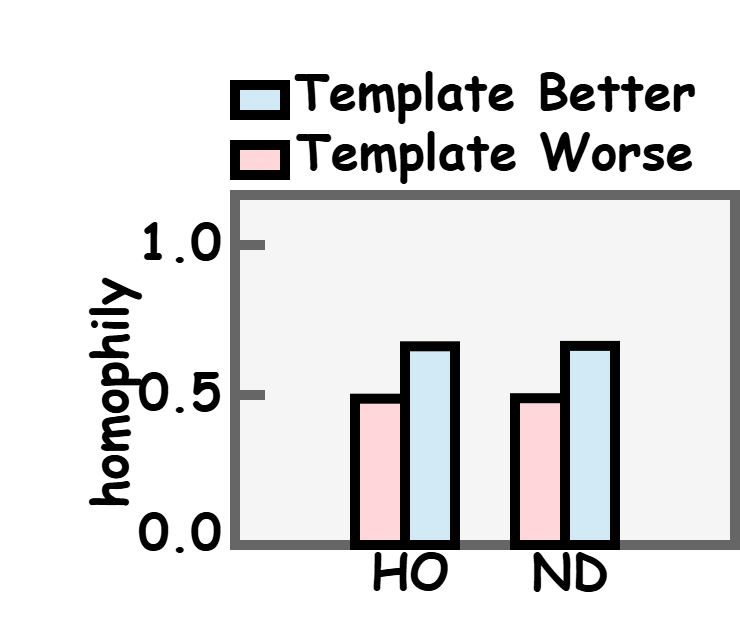}
  \end{minipage}
}
\vspace{-0.75em}
\subfigure[Cornell] {
  \begin{minipage}[b]{.142\linewidth}
    \centering
    \includegraphics[scale=0.086]{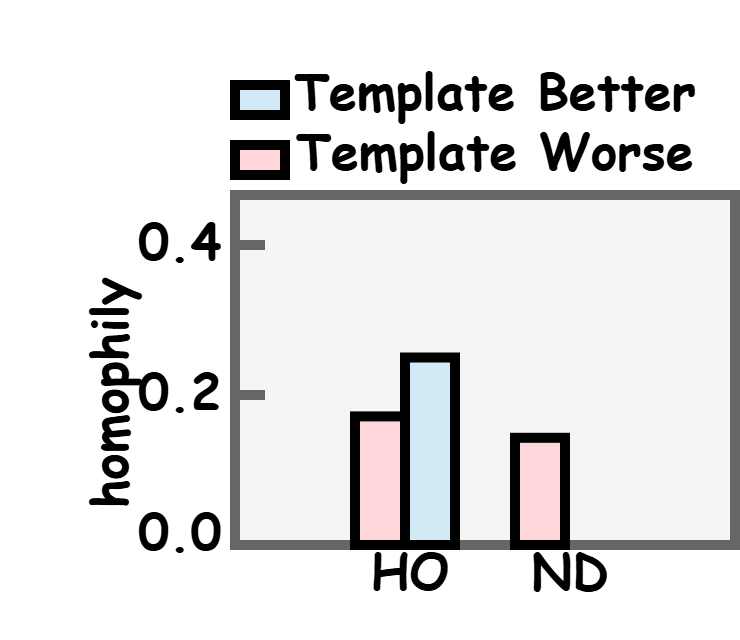}
  \end{minipage}
}
\subfigure[Texas] {
  \begin{minipage}[b]{.142\linewidth}
    \centering
    \includegraphics[scale=0.086]{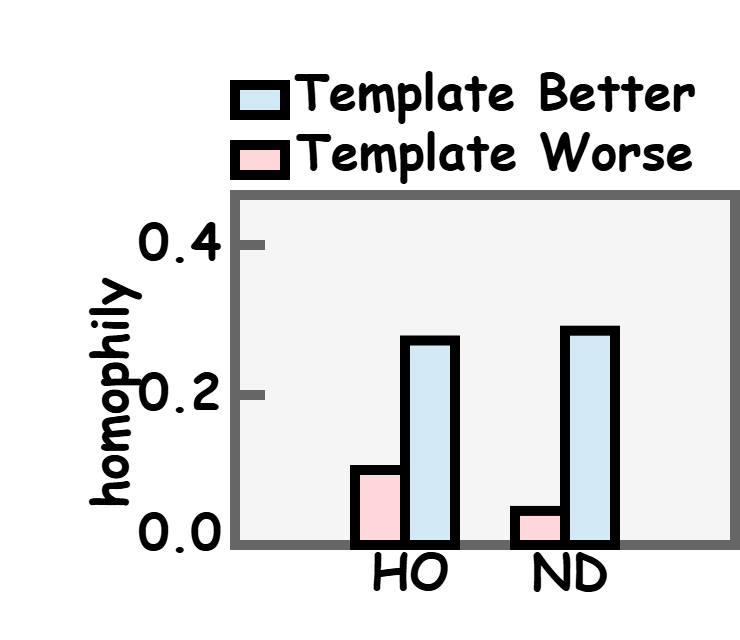}
  \end{minipage}
}
\subfigure[Wisconsin] {
  \begin{minipage}[b]{.142\linewidth}
    \centering
    \includegraphics[scale=0.086]{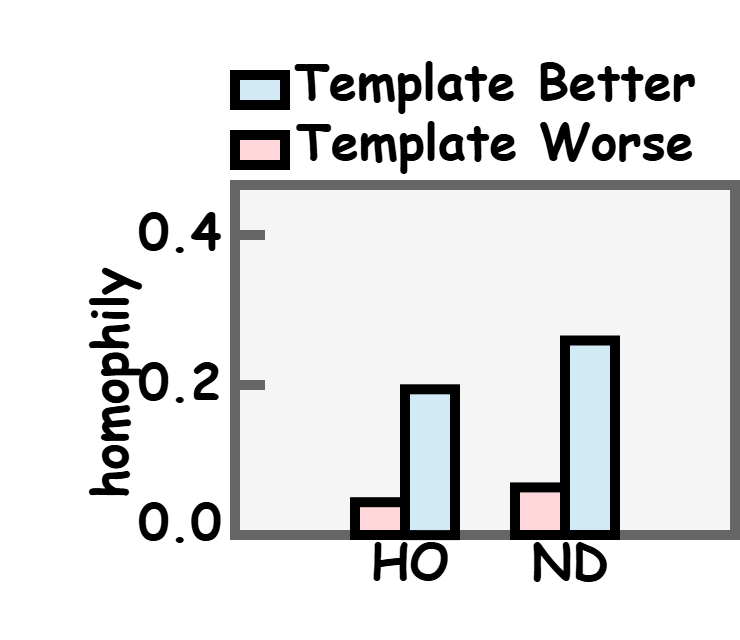}
  \end{minipage}
}
\caption{The average node-homophily for different types of nodes. "Template Better" means nodes which are predicted correctly by HO/ND but incorrectly by None, while "Template Worse" means nodes which are predicted incorrectly by HO/ND but correctly by None. }
\label{empirical_analysis}
\end{figure*}

To further investigate this phenomenon, we examine the nodes where the predictions differ between different graph token lists. Specifically, we analyze two templates (HO and ND) and whether the performance is decreased or improved. We evaluate the node-homophily (proportion of 1-hop neighbors sharing the same label as the central node) of them. As shown in Figure~\ref{empirical_analysis}, for all datasets, templates will be more likely to improve performance for nodes exhibit higher node-homophily. This further confirms that predefined templates are more beneficial for homophilic nodes but harmful for heterophilic nodes.

\subsection{Theoretical Analysis}
Next, we theoretically explore the graph token lists. We take HO as an example (the analysis of ND, which exhibits similar trends, is shown in Appendix~\ref{appendix:proof_of_ND}). Consider a graph with an average degree $n$. For a central node $u$, denote $\mathbf{T}_{u,0}^{\text{HO}} = \mathbf{x}_u$. 
The tokens are recursively calculated as $\mathbf{T}_{u,i}^{\text{HO}} = \frac{1}{n} \sum_{v \in \mathcal{N}_u} \mathbf{T}_{v,i-1}^{\text{HO}}$. To characterize nodes contributing to token $\mathbf{T}_{u,k}^{\text{HO}}$, we define $\mathbf{H}_u^k = \sum_{v \in \mathcal{N}_u^k} \mathbf{x}_v$ as the sum of the features of $\mathcal{N}_u^k$. $\mathbf{T}_{u,k}^{\text{HO}}$ can be expressed as a linear combination of $\mathbf{H}_u^0, \dots, \mathbf{H}_u^k$, i.e.,  
\begin{equation}\label{eq:M_HO_define}
\mathbf{T}_{u,k}^{\text{HO}} = \sum \nolimits_{i=0}^k \mathbf{M}_{k,i}^{\text{HO}} \mathbf{H}_u^i,
\end{equation}
where $\mathbf{M}_{k,i}^{\text{HO}}$ is the matrix capturing the contribution of $\mathcal{N}_u^i$ to $\mathbf{T}_{u,k}^{\text{HO}}$. We can obtain $\mathbf{M}_{k,i}^{\text{HO}}$ as follows:
\begin{theorem}[Recursive Properties of $\mathbf{M}_{k,i}^{\text{HO}}$]
\label{thm:MHOrecursive}
$\mathbf{M}_{k,i}^{\text{HO}}$ follows the following rules:
\begin{enumerate}[leftmargin =0.5cm]\setlength{\itemsep}{-1mm}
    \item $\mathbf{M}_{0,0}^{\text{HO}} = 1$ ($\mathbf{T}_{u,0}^{\text{HO}}$ only contains the feature of node $u$);
    \item $\mathbf{M}_{k,0}^{\text{HO}} = \mathbf{M}_{k-1,1}^{\text{HO}}$(contribution of $\mathcal{N}_u^0$ to $\mathbf{T}_{u,k}^{\text{HO}}$ equals that of $\mathcal{N}_u^1$ to $\mathbf{T}_{u,k-1}^{\text{HO}}$);
    \item $\mathbf{M}_{k,i}^{\text{HO}} = \mathbf{0}$, $i > k$ (no contribution from higher-hop neighbors than the current hop depth);
    \item $\mathbf{M}_{k,i}^{\text{HO}} = \frac{1}{n} \left( \mathbf{M}_{k-1,i-1}^{\text{HO}} + (n-1) \mathbf{M}_{k-1,i+1}^{\text{HO}} \right)
$, for $ k, i \geq 1 $.
\end{enumerate}
\end{theorem}
The proofs are shown in~\ref{appendix:MHOrecursive}. Building on Theorem~\ref{thm:MHOrecursive}, we derive the key properties and proofs of $\mathbf{M}_{k,i}^{\text{HO}}$ in Appendix~\ref{appendix:MHOprop1}-~\ref{appendix:MHOprop3} that characterize the aggregation patterns of HO. 
These properties collectively reveal that \textbf{HO contains an inherent bias of focusing more on near neighbors}. For example, even within $\mathbf{T}_{u,k}^{\text{HO}}$, the far-hop neighbors (such as the $k$-th hop) are exponentially less influential than the near-hop ones as shown in Appendix~\ref{appendix:MHOprop3}. Building on the properties of $\mathbf{M}_{i,k}^{\text{HO}}$, here we analyze how Tokenized Graph Learning Models interact with HO. 

\begin{theorem}[Effective Attention Allocation]\label{thm:MHOAttnrecursive}
Consider a simplified 1-layer Transformer model processing $\mathbf{T}_{u,0}^{\text{HO}}, \mathbf{T}_{u,1}^{\text{HO}}, \dots, \mathbf{T}_{u,L}^{\text{HO}}$ for node $u$. Let $\alpha_0, \alpha_1, \dots, \alpha_L$ ($\sum_{i=0}^L \alpha_i = 1$) denotes the attention scores from $\mathbf{T}_{u,0}^{\text{HO}}$ to all HO graph tokens. The effective attention allocated to $v \in \mathcal{N}_u^k$ is:  
\begin{equation}\label{eq:M_HO_attn_allocate}
\hat{\alpha}_k = \sum \nolimits_{i=k, \, i \equiv k \, \text{mod} \, 2}^L \alpha_i \mathbf{M}_{i,k}^{\text{HO}},
\end{equation}
The allocation exhibits two critical properties: (1) Near-Hop Dominance: for $k_1 < k_2$ with $k_1 \equiv k_2 \, \text{mod} \, 2$, $\hat{\alpha}_{k_1} > \hat{\alpha}_{k_2}$. (2) Within-Hop Indistinguishability: for any $v_1, v_2 \in \mathcal{N}_u^k$, their effective attention scores satisfy: $\hat{\alpha}_{v_1} = \hat{\alpha}_{v_2}$. 
\end{theorem}
The proof is given in Appendix~\ref{appendix:MHOAttnthm}. Furthermore, to quantify how the hop-overpriority problem of HO impacts performance, we adopt the Frobenius norm of the difference between the raw node feature and its neighbors aggregated with attention as a metric~\citep{shuman2013emerging, kalofolias2016learn}, denoted as $\| \mathbf{H}_u^0 - \hat{\mathbf{A}} \mathbf{H}^0 \|_F$, where $\hat{\mathbf{A}}$ is the attention matrix derived from Theorem~\ref{thm:MHOAttnrecursive}. 
Given the inherent bias of HO towards near-hop neighbors, we analyze how the hop-overpriority problem influences the metric:
\begin{theorem}[Smoothness Bound of Tokenized Representations]\label{thm:MHOBound}
Let $C_u^i$ denote the proportion of $\mathcal{N}_u^i$ sharing the same label as node $u$. The smoothness of node $u$'s representation satisfies:  
\begin{equation}
\label{eq:HO_smooth}
    \| \mathbf{H}_u^0 - \hat{\mathbf{A}} \mathbf{H}^0 \|_F \leq \sqrt{2} L \sum \nolimits_{i=0}^L \hat{\alpha}_i |\mathcal{N}_u^i| (1 - C_u^i),
\end{equation}
where $L$ is a Lipschitz constant.
\end{theorem} 
The proof is given in Appendix~\ref{appendix:MHOBound}. The bound clarifies how the hop-overpriority problem of HO interacts with graph homophily to influence the performance. On homophilic graphs, where $C_u^i$ is uniformly high even for near hops, the bound remains small, indicating the performance is not severely affected. 
However, heterophilic graphs exhibit low $C_u^i$ for odd hops, and $C_u^i$ increases with hop (i.e., $C_u^{i+2} > C_u^i$), but the effective attention $\hat{\alpha}_i$ decreases with hop. This leads to a critical mismatch: the hops with low $C_u^i$ receive high $\hat{\alpha}_i$, while hops with high $C_u^i$ are allocated low attention. Consequently, the bound grows significantly, limiting the ability of the models to learn meaningful representations, which aligns with our empirical results.
\section{Method}\label{sec:method}
In this section, we introduce our method in details. The overall framework is shown in Figure~\ref{fig:pipeline}. \MethodName~s a simple-yet-effective plug-and-play module that can be easily compatible with various tokenized graph learning methods such as different variants of Graph LLMs and Graph Transformers. 
\begin{figure}[t]
  \centering
  \includegraphics[width=1.0\textwidth]{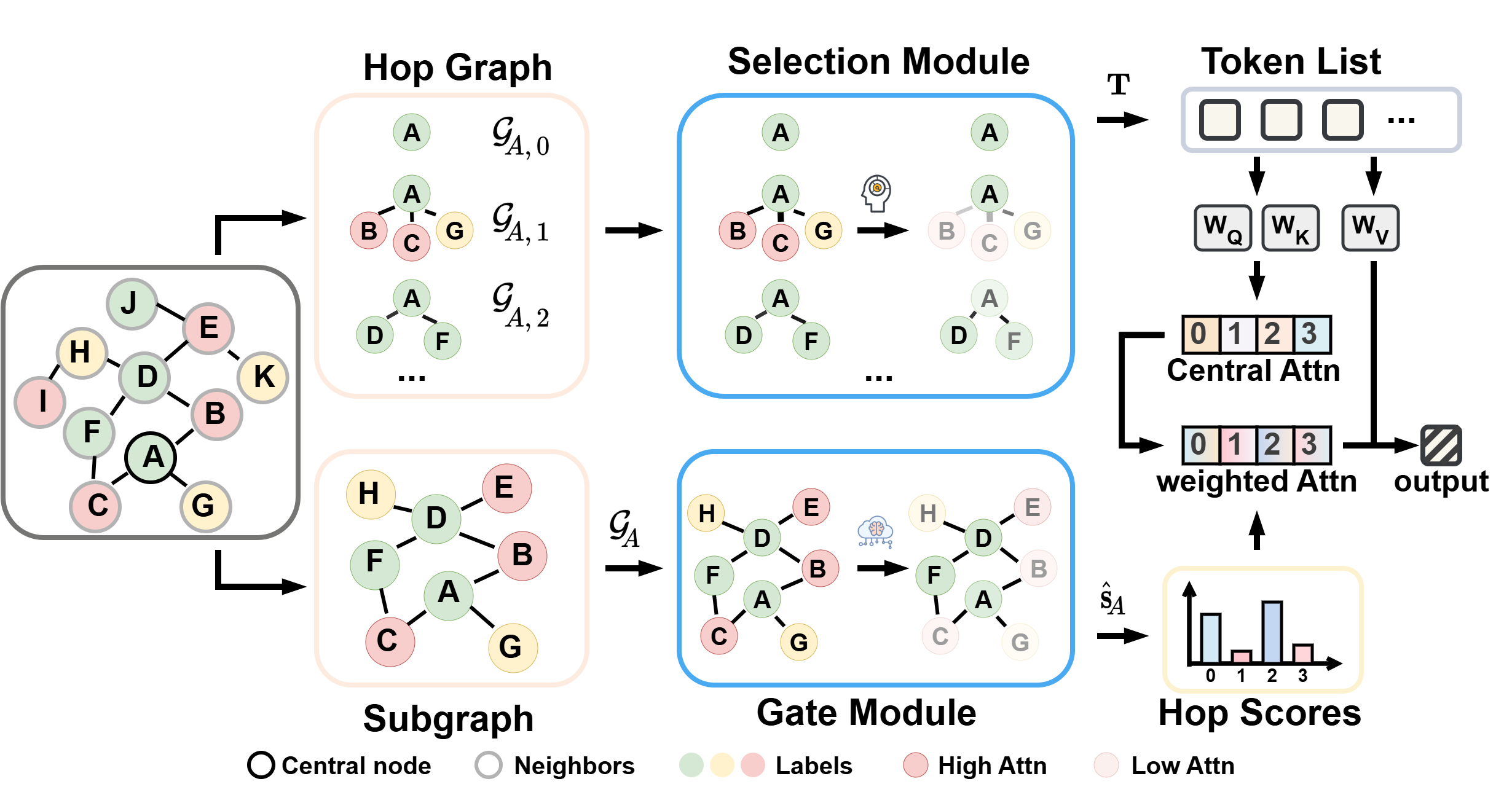}
    \vspace{-0.25cm}
    \caption{The overall framework of \MethodName, including a \textit{gate module} which learns hop scores from the central node's subgraph to rebalance attention and mitigate hop-overpriority problem, and a \textit{selection module} which constructs hop subgraphs, computes within-hop node attention, and aggregates features into tokens. These tokens form a list input to TGLMs; raw attention scores are adjusted by hop weights to produce task-adaptive representations for homophilic and heterophilic graphs.}
    \label{fig:pipeline}
    \vspace{-0.5cm}
\end{figure}
\subsection{\MethodName: Learnable Token List for Tokenized Graph Learning Models}
Inspired by our preliminary experiments and theoretical analysis, 
we propose \MethodName, a learnable token list framework which flexibly adjusts the priority of hops and handles the features of nodes from different hops independently to focus on task-relevant signals. Specifically, \MethodName~adaptively allocates the attention in token lists by integrating the gate module and the selection module, enabling adaptive focus on task-relevant nodes across both homophilic and heterophilic datasets.

Given the total number of hops \( L \) and the size of neighbor sampling \( n_{i} \) (where \( i \in \{1, 2, \dots, L\} \)), 
we adopt a gate module to flexibly assign scores to each hop.
The gate module processes the subgraph \( \mathcal{G}_u \) of the central node \( u \) and learns context-aware embeddings leveraging a Graph Attention Network (GAT)~\citep{velivckovic2017graph}. To derive hop-specific weights, we first use the embedding of node $u$ from the gate module as raw scores \( \mathbf{s}_u^{\text{raw}} \in \mathbb{R}^{L+1} \), which corresponds to the importance of different hops of neighbors. The scores are normalized through the softmax function to obtain the weights:  
\begin{equation}
    \hat{\mathbf{s}}_u = \text{Softmax}(\mathbf{s}_u^{\text{raw}}),
\end{equation}
where \( \hat{\mathbf{s}}_{u,i} \) denotes the weight assigned to the \( i \)-th hop. Intuitively, higher \( \hat{\mathbf{s}}_{u,i} \) indicates the $i$-th hop is more important for the current tasks, which adaptively mitigates the hop-overpriority problem by re-balancing attention across the hops.

After obtaining the weights, we next construct the token list. For the $i$-th hop, \MethodName~constructs a hop token \( \mathbf{T}_i \) by aggregating the features of the nodes in \( \mathcal{N}_u^i \). Specifically, for the 0-th hop, the token is simply the raw feature of node, i.e., \( \mathbf{T}_0 = \mathbf{x}_u \). For the \( i \)-th hop (\( i \geq 1 \)), \MethodName~samples \( n_{i} \) nodes from \( \mathcal{N}_u^i \), appends a self-loop edge, and forms a subgraph \( \mathcal{G}_u^i \) with \( n_{i} + 1 \) edges. A GAT layer is adopted as the selection module to process \( \mathcal{G}_u^i \) and calculate the attention scores \( \beta_{u,i,v} \) between \( u \) and each $v$ sampled from \(\mathcal{G}_u^i \). The \( i \)-th hop token is then obtained by weighted aggregation:  
\begin{equation}
    \mathbf{T}_i = \sum \nolimits_{v \in \mathcal{G}_u^i} \beta_{u,i,v} \cdot \mathbf{x}_v,
\end{equation}
where \( \sum_{v \in \mathcal{G}_u^i} \beta_{u,i,v} = 1 \).  Then, the token list \( \mathbf{T} = [\mathbf{T}_0, \mathbf{T}_1, \dots, \mathbf{T}_L] \) is inputted into TGLMs to obtain the output of value matrix $\mathbf{T}\mathbf{W}_V$ and calculate the attention scores \( \alpha_u = [\alpha_{u,0}, \alpha_{u,1}, \dots, \alpha_{u,L}] \) of the central node $u$, where \( \alpha_{u,i} \) reflects the focus on the \( i \)-th hop token for node $u$. To align the attention with the scores we obtain from the gate module, \MethodName~adjusts \( \alpha_u \) using \( \hat{\mathbf{s}}_u \):  
\begin{equation}
    \hat{\alpha}_{u,i} = \frac{\alpha_{u,i} \cdot \hat{\mathbf{s}}_{u,i}}{\langle \alpha_u \cdot \hat{\mathbf{s}}_u \rangle},
\end{equation}
where \( \hat{\alpha}_{u,i} \) is the adjusted attention score for the $i$-th hop. The final node representation is obtained by aggregating the tokens with the adjusted scores:  
\begin{equation}
\mathbf{z}_u = \sum \nolimits_{i=0}^L \hat{\alpha}_{u,i} \cdot \mathbf{T}_i\mathbf{W}_V,
\end{equation}
In sum, by adaptively weighting hops via \( \hat{\mathbf{s}}_u \) and refining the aggregation process of features via the selection module, \MethodName~adapts to homophilic and heterophilic graphs, ensuring that the task-relevant nodes hold a leading position of the aggregation process. 

\subsection{Theoreical Analysis}

To verify the design of \MethodName, we present theoretical properties that explain its adaptability and generality over predefined token lists. A key strength of \MethodName~is its flexibility to accommodate various tokenization strategies by adjusting its learnable components. Specifically, \MethodName~can recover the behavior of existing predefined templates (e.g., HO and ND) through parameter specialization, highlighting its ability to generalize prior approaches. We formalize the following theorem:

\begin{theorem}\label{thm:Generalization}
\MethodName~generalizes pre-defined token lists HO and ND as special cases.
\end{theorem}
\begin{proof}
For HO, the attention to \( k \)-th hop nodes is \( \hat{\alpha}_k = \sum_{i=k, \, i \equiv k \, \text{mod} \, 2}^L \alpha_i \mathbf{M}_{i,k}^{\text{HO}} \), where \( \mathbf{M}_{i,k}^{\text{HO}} \) is the hop contribution matrix. For \MethodName, setting $\mathcal{G}_u^i=\mathcal{N}_u^i$, \( \beta_{u,k,v} = \frac{1}{|\mathcal{N}_u^k|} \) ,and $\hat{\alpha}_{u,k} = \frac{\hat{\alpha}_k}{\beta_{u,k,v}}$ leads to 
\begin{equation}
s_{u,k}\propto \frac{\hat{\alpha}_k}{\beta_{u,k,v}\cdot \alpha_{u,i}},
\end{equation}
This can recover the attention of HO. Thus, HO is a special case of \MethodName~with uniform within-hop attention and fixed gate weights. 
The proof of ND is provided in Appendix~\ref{relationship_with_ND}.\end{proof}

Furthermore, to quantitatively analyze how our method tackles the hop-overpriority problem, 
we re-examine the bound in Theorem~\ref{thm:MHOBound}. 
\begin{theorem}
The norm of \MethodName~is bounded by:  
\begin{equation}
\|\mathbf{H}_u^0 - \hat{\mathbf{A}} \mathbf{H}^0\|_F \leq \sqrt{2}L \cdot \frac{\sum_{i=0}^{L} \hat{\mathbf{s}}_{u,i} |\mathcal{G}_u^i| (1 - C_{u}^{i})}{\sum_{i=0}^{L} |\mathcal{G}_u^i| (1 - C_{u}^{i}) + \frac{\eta_u}{\gamma_u} \sum_{i=0}^{L} |\mathcal{G}_u^i| C_{u}^{i}},
\end{equation}
where $\gamma_u = \mathbb{E}_{v \in \mathbf{N}_u^i, \mathbf{y}_u = \mathbf{y}_v} \exp\left(\frac{\mathbf{q}_u \mathbf{k}_v^\top}{\sqrt{h}}\right)$, and $\eta_u = \mathbb{E}_{v \in \mathbf{N}_u^i, \mathbf{y}_u \neq \mathbf{y}_v} \exp\left(\frac{\mathbf{q}_u \mathbf{k}_v^\top}{\sqrt{h}}\right)$ are constants.    
\end{theorem}
Detailed analysis and proofs are in Appendix~\ref{smooth_bound_of_our_method}. This theorem shows that \MethodName~minimizes the error by assigning higher \( \hat{\mathbf{s}}_{u,i} \) to hops with higher \( C_{u}^{i} \), which is critical for heterophilic graphs, where predefined token lists pay more attention to inconsistent hops according to Theorem~\ref{thm:MHOBound}.

\section{Experiments}\label{sec:experiments}
In this section, we conduct experiments to answer the following research questions. \textbf{Q1}: Is \MethodName~capable of augmenting existing LLMs for graphs for text-attributed graphs? \textbf{Q2}: Is \MethodName~effective in enhancing graph Transformers without texts? \textbf{Q3}: How does each component contribute to \MethodName?
\subsection{Results on Text-attributed Graphs}\label{sec:exp-tags}
To answer \textbf{Q1}, we first conduct experiments on text-attributed graphs.

\textbf{Experimental Setups}: we choose LLaGA~\cite{chen2024llaga}, a representative Graph LLM, as our backbone. Specifically, we replace the token list of LLaGA with \MethodName~and keep other parts unchanged. Besides comparing LLaGA with different token lists, i.e., HO/HD, we also compare \MethodName~with four additional baselines, including two classical GNNs, GCN~\citep{kipf2016semi} and GAT~\citep{velivckovic2017graph}, and one GNN for heterophilic graphs, H2GCN~\citep{zhu2020beyond}, and one representative Graph Transformer, NodeFormer~\citep{wu2022NodeFormer}. For datasets, we follow~\citep{chen2024llaga} and use two homophilic datasets, Cora and PubMed. For heterophilic datasets, we use Cornell, Texas, Wisconsin, and Actor by collecting the texts of all nodes and relations between them. More details are provided in Appendix~\ref{appendix:datasets_llaga}. 

\textbf{Results}: Table~\ref{main-llaga} presents the results for text-attributed graphs. LLaGA-\MethodName~consistently outperforms both classical GNNs and LLaGA variants with original token lists across node classification and link prediction tasks. For node classification, LLaGA-\MethodName~achieves the best accuracy on all six datasets, with an average improvement of 10.39\% over LLaGA-HO and LLaGA-ND. In particular, its superiority is more pronounced on heterophilic datasets. For link prediction, LLaGA-\MethodName~also leads, achieving an average gain of ~4.67\% over the second-best baseline across all six datasets. This consistent improvement underscores its capability to better align attention with task-relevant edges.

\begin{table}[t]
  \caption{The results on six text-attributed graphs with LLaGA as the backbone, where \textbf{bold} signifies the best result and {\ul underline} highlights the second best result.}
  \label{main-llaga}
  \resizebox{\textwidth}{!}{%

\begin{tabular}{c|c|cccccc}
\toprule
Task                                                                           & Model      & Cora           & PubMed         & Cornell        & Texas          & Wisconsin      & Actor          \\ \midrule
\multirow{7}{*}{\begin{tabular}[c]{@{}c@{}}Node\\ Classification\end{tabular}} & GCN & 88.93\scalebox{0.75}{±0.12} & 92.96\scalebox{0.75}{±0.15} & 40.00\scalebox{0.75}{±3.12} & 56.13\scalebox{0.75}{±2.89} & 45.83\scalebox{0.75}{±3.15} & 70.57\scalebox{0.75}{±0.34} \\
& GAT & 88.97\scalebox{0.75}{±0.14} & 92.33\scalebox{0.75}{±0.18} & 36.67\scalebox{0.75}{±4.13} & 56.77\scalebox{0.75}{±2.24} & 43.75\scalebox{0.75}{±3.60} & 69.11\scalebox{0.75}{±0.36} \\
& H2GCN & 88.82\scalebox{0.75}{±0.11} & 93.61\scalebox{0.75}{±0.13} & {\ul 58.67\scalebox{0.75}{±2.28}} & {\ul 82.58\scalebox{0.75}{±0.79}} & {\ul 69.58\scalebox{0.75}{±1.95}} & 74.62\scalebox{0.75}{±0.40} \\
& NodeFormer & 88.23\scalebox{0.75}{±0.17} & 94.90\scalebox{0.75}{±0.19} & 55.33\scalebox{0.75}{±3.34} & 81.29\scalebox{0.75}{±1.25} & 65.83\scalebox{0.75}{±2.62} & 76.23\scalebox{0.75}{±0.42} \\ \cmidrule{2-8}
& LLaGA-HO & {\ul 89.22\scalebox{0.75}{±0.10}} & {\ul 95.03\scalebox{0.75}{±0.12}} & 42.67\scalebox{0.75}{±4.38} & 63.23\scalebox{0.75}{±2.97} & 49.58\scalebox{0.75}{±3.74} & 77.05\scalebox{0.75}{±0.41} \\
& LLaGA-ND & 88.86\scalebox{0.75}{±0.13} & {\ul 95.03\scalebox{0.75}{±0.14}} & 46.67\scalebox{0.75}{±4.38} & 74.19\scalebox{0.75}{±1.91} & 50.83\scalebox{0.75}{±3.60} & {\ul 77.34\scalebox{0.75}{±0.39}}  \\ \cmidrule{2-8}
& LLaGA-\MethodName~& \textbf{89.30\scalebox{0.75}{±0.09}} & \textbf{95.18\scalebox{0.75}{±0.11}} & \textbf{64.67\scalebox{0.75}{±1.21}} & \textbf{90.32\scalebox{0.75}{±0.68}} & \textbf{77.08\scalebox{0.75}{±0.79}} & \textbf{79.04\scalebox{0.75}{±0.37}} \\ \midrule
\multirow{7}{*}{\begin{tabular}[c]{@{}c@{}}Link\\ Prediction\end{tabular}} & GCN & 81.24\scalebox{0.75}{±0.21} & 90.50\scalebox{0.75}{±0.23} & 66.52\scalebox{0.75}{±2.29} & {\ul 74.00\scalebox{0.75}{±1.78}} & 71.49\scalebox{0.75}{±1.57} & 74.55\scalebox{0.75}{±0.44} \\
& GAT & 79.68\scalebox{0.75}{±0.23} & 88.67\scalebox{0.75}{±0.25} & 65.22\scalebox{0.75}{±2.36} & 69.20\scalebox{0.75}{±1.86} & {\ul 73.26\scalebox{0.75}{±1.65}} & 74.82\scalebox{0.75}{±0.46} \\
& H2GCN & 80.24\scalebox{0.75}{±0.20} & 88.03\scalebox{0.75}{±0.22} & {\ul 70.43\scalebox{0.75}{±1.32}} & 72.80\scalebox{0.75}{±1.81} & 72.56\scalebox{0.75}{±1.60} & 75.12\scalebox{0.75}{±0.45} \\
& NodeFormer & 78.12\scalebox{0.75}{±0.24} & 79.38\scalebox{0.75}{±0.26} & 57.83\scalebox{0.75}{±2.40} & 64.40\scalebox{0.75}{±1.93} & 68.14\scalebox{0.75}{±1.71} & 64.62\scalebox{0.75}{±0.47} \\ \cmidrule{2-8}
& LLaGA-HO & 81.17\scalebox{0.75}{±0.19} & 89.72\scalebox{0.75}{±0.21} & 62.17\scalebox{0.75}{±1.39} & 71.40\scalebox{0.75}{±1.92} & 65.00\scalebox{0.75}{±1.70} & 86.23\scalebox{0.75}{±0.42} \\
& LLaGA-ND & {\ul 82.29\scalebox{0.75}{±0.18}} & {\ul 91.31\scalebox{0.75}{±0.20}} & 63.04\scalebox{0.75}{±1.35} & 71.60\scalebox{0.75}{±1.89} & 64.44\scalebox{0.75}{±1.67} & {\ul 86.44\scalebox{0.75}{±0.41}} \\ \cmidrule{2-8}
& LLaGA-\MethodName~& \textbf{83.82\scalebox{0.75}{±0.16}} & \textbf{91.84\scalebox{0.75}{±0.18}} & \textbf{71.30\scalebox{0.75}{±0.74}} & \textbf{76.80\scalebox{0.75}{±0.72}} & \textbf{73.95\scalebox{0.75}{±0.93}} & \textbf{89.48\scalebox{0.75}{±0.38}} \\ \bottomrule
\end{tabular}
}
\vspace{-0.7em}
\end{table}

\subsection{Results on Benchmarks for Graph Transformers}
To answer \textbf{Q2}, we further conduct experiments on benchmarks without text for Graph Transformers.

\textbf{Experimental Setups}: we choose two classical Graph Transformers for evaluation: NAGphormer~\citep{chen2022nagphormer}, and VCR-Graphormer~\citep{fu2024vcr}. Similar to Graph LLMs, we replace the token list of NAGphormer and VCR-Graphormer with \MethodName~and keep other parts unchanged. Other baselines are the same as Section~\ref{sec:exp-tags}. For datasets, we follow NAGphormer and VCR-Graphormer by using PubMed~\citep{namata2012query}, Computers, and Photo~\citep{shchur2018pitfalls} with numerical node features. Additionally, we adopt four heterophilic datasets (Cornell, Texas, Wisconsin, and Actor~\citep{pei2020geom}). More details are provided in Appendix~\ref{appendix:datasets_gt}.

\begin{table}[t]
  \caption{The results on seven graph benchmark datasets without using text with NAGphormer and VCR-Graphormer as the backbone, where \textbf{bold} signifies the best result.}
  \label{main-gt}
  \resizebox{\textwidth}{!}{%

\begin{tabular}{c|ccccccc}
\toprule
Model               & PubMed         & Computers      & Photo          & Cornell        & Texas          & Wisconsin      & Actor          \\ \midrule
GCN & 86.54\scalebox{0.75}{±0.21} & 89.65\scalebox{0.75}{±0.24} & 92.70\scalebox{0.75}{±0.27} & 47.37\scalebox{0.75}{±3.25} & 52.63\scalebox{0.75}{±2.74} & 54.09\scalebox{0.75}{±2.52} & 29.87\scalebox{0.75}{±1.03} \\
GAT & 86.32\scalebox{0.75}{±0.23} & 90.78\scalebox{0.75}{±0.26} & 93.87\scalebox{0.75}{±0.29} & 44.74\scalebox{0.75}{±4.31} & 55.26\scalebox{0.75}{±2.82} & 52.94\scalebox{0.75}{±3.60} & 29.08\scalebox{0.75}{±1.41} \\
H2GCN & 88.49\scalebox{0.75}{±0.19} & 89.86\scalebox{0.75}{±0.22} & 94.95\scalebox{0.75}{±0.25} & 63.16\scalebox{0.75}{±2.28} & 65.79\scalebox{0.75}{±1.79} & 66.67\scalebox{0.75}{±1.55} & 34.74\scalebox{0.75}{±0.69} \\
NodeFormer & 88.89\scalebox{0.75}{±0.20} & 90.96\scalebox{0.75}{±0.23} & 95.02\scalebox{0.75}{±0.26} & 60.53\scalebox{0.75}{±2.34} & 65.79\scalebox{0.75}{±1.85} & 62.75\scalebox{0.75}{±2.62} & 34.87\scalebox{0.75}{±0.80} \\ \midrule
NAGphormer & 89.55\scalebox{0.75}{±0.18} & 91.22\scalebox{0.75}{±0.21} & 95.49\scalebox{0.75}{±0.24} & 55.26\scalebox{0.75}{±1.38} & 63.16\scalebox{0.75}{±1.91} & 62.75\scalebox{0.75}{±1.68} & 34.61\scalebox{0.75}{±0.67} \\
NAGphormer-\MethodName~& \textbf{90.11\scalebox{0.75}{±0.10}} & \textbf{91.78\scalebox{0.75}{±0.12}} & \textbf{96.01\scalebox{0.75}{±0.15}} & \textbf{65.79\scalebox{0.75}{±1.21}} & \textbf{73.68\scalebox{0.75}{±1.68}} & \textbf{81.57\scalebox{0.75}{±1.49}} & \textbf{36.84\scalebox{0.75}{±0.51}} \\ \midrule
VCR-Graphormer & 88.82\scalebox{0.75}{±0.17} & 90.51\scalebox{0.75}{±0.20} & 95.53\scalebox{0.75}{±0.23} & 52.63\scalebox{0.75}{±1.29} & 65.79\scalebox{0.75}{±1.78} & 60.78\scalebox{0.75}{±1.57} & 35.59\scalebox{0.75}{±0.36} \\
VCR-Graphormer-\MethodName~& \textbf{89.45\scalebox{0.75}{±0.12}} & \textbf{91.13\scalebox{0.75}{±0.14}} & \textbf{95.82\scalebox{0.75}{±0.17}} & \textbf{68.42\scalebox{0.75}{±1.24}} & \textbf{73.68\scalebox{0.75}{±1.72}} & \textbf{79.22\scalebox{0.75}{±1.53}} & \textbf{38.03\scalebox{0.75}{±0.42}} \\ \bottomrule
\end{tabular}
}
\vspace{-1.25em}
\end{table}

\textbf{Results}: The results on benchmarks for Graph Transformers are shown in Table~\ref{main-gt}. When integrating \MethodName~into NAGphormer and VCR-Graphormer, both models exhibit significant improvements over their original versions and classical baselines. For instance, NAGphormer-\MethodName~outperforms the original NAGphormer by 0.55\% average on homophilic datasets, while achieving 10.53\% average on heterophilic datasets. Similarly, VCR-Graphormer-\MethodName~shows a 11.14\% average improvement on heterophilic datasets, with top performance on all seven benchmarks. These results confirm the effectiveness of \MethodName~in enhancing TGLMs, particularly under heterophily, aligning with our analysis of the hop-overpriority problem.


\subsection{Ablation Studies \& Analysis}\label{sec:ablation_studies_analysis}
To answer \text{Q3}, next we conduct ablation studies and detailed analyses for each component. 

\textbf{Ablation Studies.}
First, we carry out ablation studies to evaluate the gate module and the selection module of \MethodName. Specifically, we compare two variants: ``w/o gate'' denotes removing the gate module by setting same weights for all hops. ``w/o selection'' indicates removing the selection module by letting the neighbors of the central node share the same attention score. The results are shown in Table~\ref{ablation}. We can observe that: (1) ``w/o gate'' underperforms \MethodName~on all datasets, demonstrating its effectiveness in focusing on task-relevant hops; (2) ``w/o selection'' also lags behind \MethodName, indicating the effectiveness of the selection module in distinguishing critical within-hop nodes; (3) the performance decrease is more severe on heterophilic datasets, indicating the indispensable role of hop-level (gate) and within-hop (selection) mechanisms in adapting to heterophilic graphs.

\begin{table}[t]
  \caption{The results of ablation studies for the gate and selection module of \MethodName.}
  \label{ablation}
  \centering
  \vspace{-0.8em}
  \resizebox{0.65\textwidth}{!}{%
\begin{tabular}{c|cccccc}
\midrule
Model                              & \begin{tabular}[c]{@{}c@{}}LLaGA\\ Cora\end{tabular} & \begin{tabular}[c]{@{}c@{}}LLaGA\\ Texas\end{tabular} & \begin{tabular}[c]{@{}c@{}}NAG\\ PubMed\end{tabular} & \begin{tabular}[c]{@{}c@{}}NAG\\ Actor\end{tabular} & \begin{tabular}[c]{@{}c@{}}VCR\\ PubMed\end{tabular} & \begin{tabular}[c]{@{}c@{}}VCR\\ Actor\end{tabular} \\ \midrule
w/o gate                           & {\ul 89.14}                                                 & {\ul 80.65}                                                  & 89.40                                                 & {\ul 35.66}                                                & {\ul 89.02}                                                 & {\ul 36.83}                                                \\
\multicolumn{1}{l|}{w/o selection} & 89.11                                                 & 70.97                                                  & {\ul 89.58}                                                 & 35.07                                                & 88.89                                                 & 35.99                                                \\ 
\MethodName(full)                              & \textbf{89.30}                                        & \textbf{90.32}                                         & \textbf{90.11}                                        & \textbf{36.84}                                       & \textbf{89.45}                                        & \textbf{38.03}    \\ \bottomrule
\end{tabular}}
\vspace{-0.4cm}
\end{table}

\textbf{Analysis of the gate module.}
Next, 
we analyze the scores the gate module assigns to different hops for a more fine-grained analysis. Intuitively, hops with higher node-homophily indicate nodes with the same label and thus the gate module should allocate higher attention scores. As shown in Figure~\ref{fig:gate}, on homophilic graphs, the gate module allocates the highest scores to ``hop 1'', followed by ``hop 2'', while on heterophilic graphs, the gate module instead focuses more on ``hop 2'' and the central node. The results indicate that the gate module assign larger weights to hops with higher task-relevance.

\begin{figure*}[t]
    \centering
    \subfigure[Homophilic Graphs]
    {
        \begin{minipage}[b]{0.4\textwidth}
            \centering
            \includegraphics[width=\linewidth]{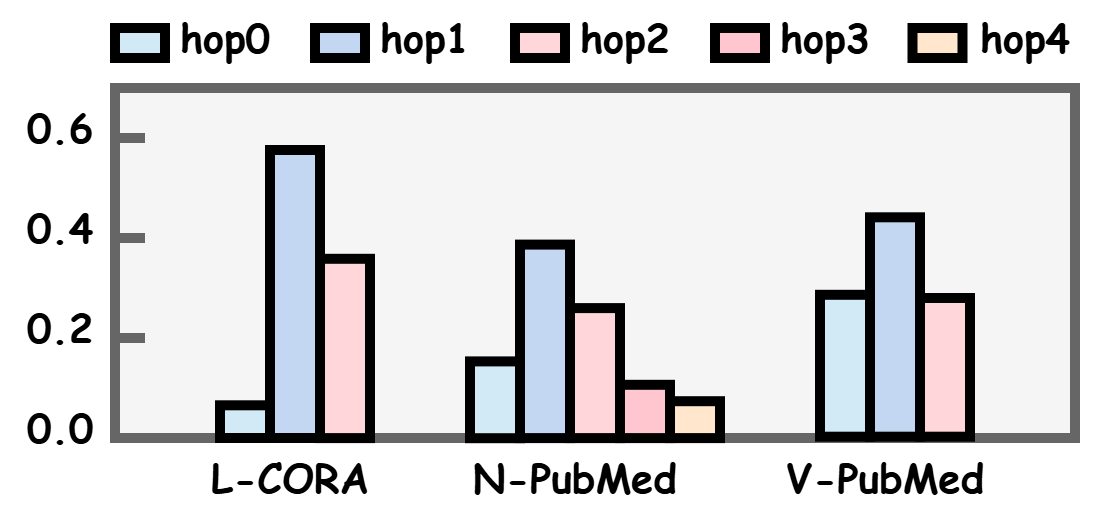}
        \end{minipage}
    }
    \subfigure[Heterophilic Graphs]
    {
        \begin{minipage}[b]{0.4\textwidth}
            \centering
            \includegraphics[width=\linewidth]{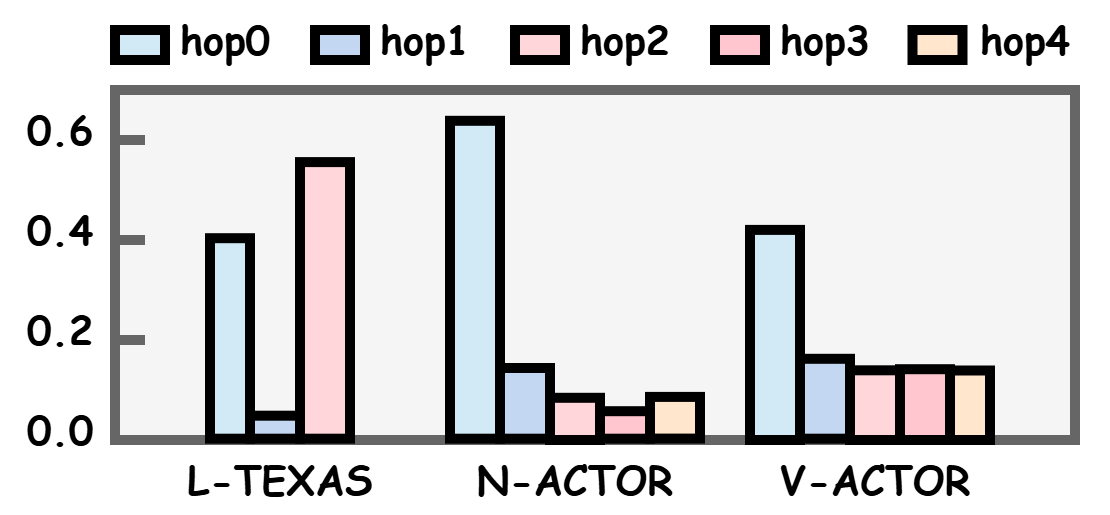}
        \end{minipage}
    }
\vspace{-0.3cm}
\caption{The analysis of the score by the gate module vs. the number of hops. ``L'', ``N'', and ``V'' indicates abbreviation for LLaGA, NAGphormer, and VCR-Graphormer, respectively.}
\label{fig:gate}
\vspace{-0.5cm}
\end{figure*}

\begin{wrapfigure}{r}{0.55\textwidth}
\vspace{-1.75em}
  \centering
  \includegraphics[width=0.4\textwidth]{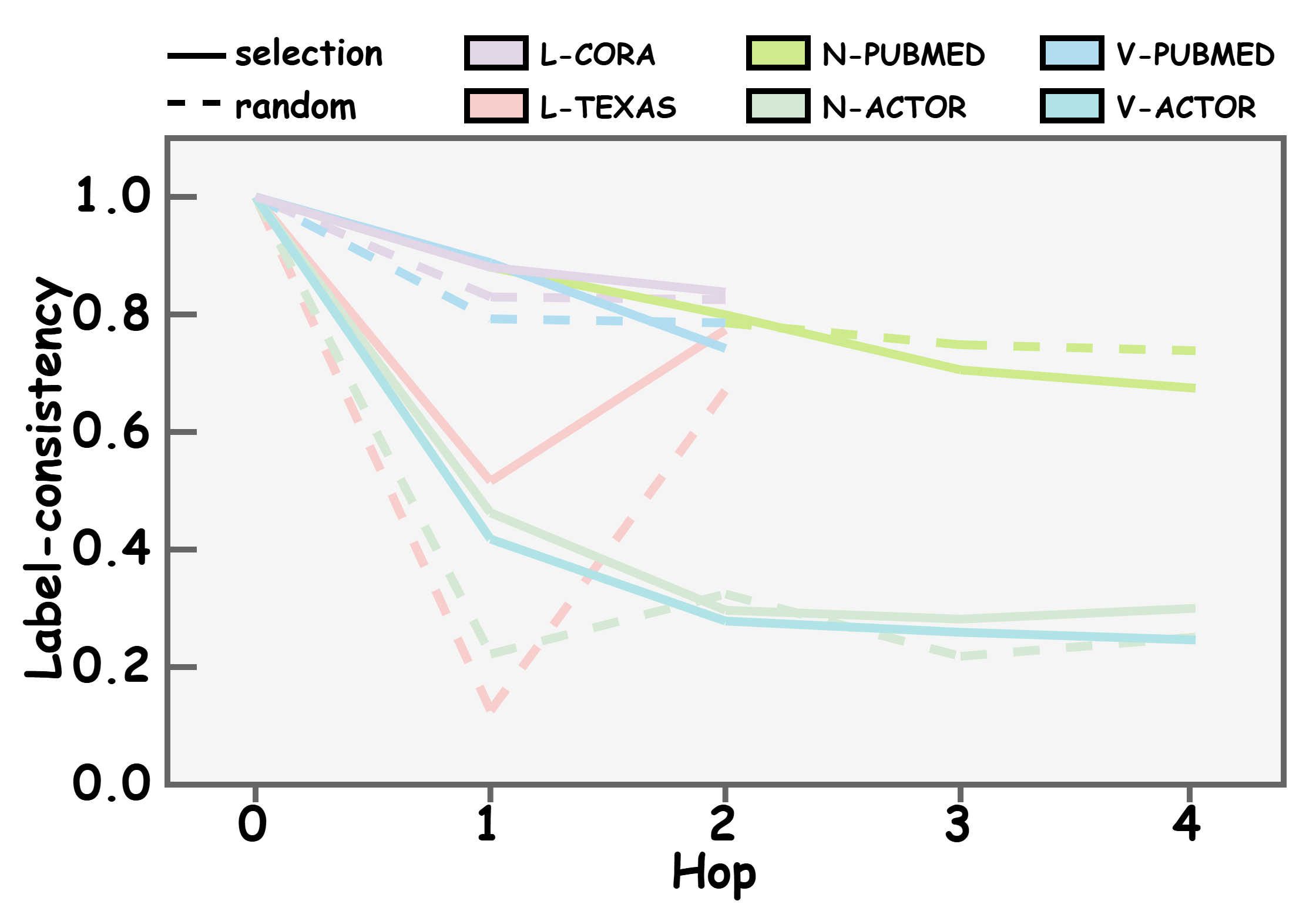}
    \vspace{-0.4cm}
    \caption{The label-consistency of the selection module and node-homophily of hops. ``L'', ``N'', and ``V'' indicates abbreviation for LLaGA, NAGphormer, and VCR-Graphormer, respectively.} 
    \label{fig:selection}
    \vspace{-1.0em}
\end{wrapfigure}

\textbf{Analysis of the selection module.} Lastly, we aim to analyze whether the selection module can select nodes with the same label as the central nodes. Therefore, we analyze the label-consistency of nodes with the highest attention scores within each hop. As a reference line, we compare with randomly selecting nodes from each hop. As shown in Figure~\ref{fig:selection}, the label-consistency of the selected nodes exceeds the random baseline, indicating that the selection module effectively identifies and prioritizes critical nodes. We provide more experiments and analysis of \MethodName~in Appendix~\ref{more_analysis}. 

\vspace{-0.15cm}
\section{Conclusion}\label{sec:conclusion}
\vspace{-0.2cm}
In this paper, we first identify the hop-overpriority problem for predefined token lists in TGLMs.
Then, we propose Learnable Graph Token List (\MethodName), an adaptive framework that adjusts hop weights and prioritizes informative nodes within and across hops, 
enhancing the adaptability on both homophilic and heterophilic graphs. We also theoretically show that \MethodName~can effective address the hop-overpriority problem. 
Experiments across diverse TGLM backbones demonstrate that \MethodName~consistently improves performance and mitigate the hop-overpriority problem.

{\small
\bibliographystyle{unsrtnat}
\bibliography{ref}
}

{
\small
}

\section*{NeurIPS Paper Checklist}

The checklist is designed to encourage best practices for responsible machine learning research, addressing issues of reproducibility, transparency, research ethics, and societal impact. Do not remove the checklist: {\bf The papers not including the checklist will be desk rejected.} The checklist should follow the references and follow the (optional) supplemental material.  The checklist does NOT count towards the page limit. 

Please read the checklist guidelines carefully for information on how to answer these questions. For each question in the checklist:
\begin{itemize}
    \item You should answer \answerYes{}, \answerNo{}, or \answerNA{}.
    \item \answerNA{} means either that the question is Not Applicable for that particular paper or the relevant information is Not Available.
    \item Please provide a short (1–2 sentence) justification right after your answer (even for NA). 
\end{itemize}

{\bf The checklist answers are an integral part of your paper submission.} They are visible to the reviewers, area chairs, senior area chairs, and ethics reviewers. You will be asked to also include it (after eventual revisions) with the final version of your paper, and its final version will be published with the paper.

The reviewers of your paper will be asked to use the checklist as one of the factors in their evaluation. While "\answerYes{}" is generally preferable to "\answerNo{}", it is perfectly acceptable to answer "\answerNo{}" provided a proper justification is given (e.g., "error bars are not reported because it would be too computationally expensive" or "we were unable to find the license for the dataset we used"). In general, answering "\answerNo{}" or "\answerNA{}" is not grounds for rejection. While the questions are phrased in a binary way, we acknowledge that the true answer is often more nuanced, so please just use your best judgment and write a justification to elaborate. All supporting evidence can appear either in the main paper or the supplemental material, provided in appendix. If you answer \answerYes{} to a question, in the justification please point to the section(s) where related material for the question can be found.


\begin{enumerate}

\item {\bf Claims}
    \item[] Question: Do the main claims made in the abstract and introduction accurately reflect the paper's contributions and scope?
    \item[] Answer: \answerYes{} 
    \item[] Justification: The abstract and introduction accurately reflect the paper's contributions and scope. 

\item {\bf Limitations}
    \item[] Question: Does the paper discuss the limitations of the work performed by the authors?
    \item[] Answer: \answerYes{} 
    \item[] Justification: We provide detailed discussions about limitations in Appendix~\ref{limitations}.

\item {\bf Theory assumptions and proofs}
    \item[] Question: For each theoretical result, does the paper provide the full set of assumptions and a complete (and correct) proof?
    \item[] Answer: \answerYes{} 
    \item[] Justification: We provide complete theory assumptions and proofs of pre-defined token list in Appendix~\ref{appendix:proof} and Appendix~\ref{appendix:proof_of_ND}. Furthermore, we provide the  theory assumptions and proofs of \MethodName~in Appendix~\ref{method_theoretical}.

    \item {\bf Experimental result reproducibility}
    \item[] Question: Does the paper fully disclose all the information needed to reproduce the main experimental results of the paper to the extent that it affects the main claims and/or conclusions of the paper (regardless of whether the code and data are provided or not)?
    \item[] Answer: \answerYes{} 
    \item[] Justification: Yes, we disclose all the information needed to reproduce the main experimental results of the paper in Appendix~\ref{environment}.

\item {\bf Open access to data and code}
    \item[] Question: Does the paper provide open access to the data and code, with sufficient instructions to faithfully reproduce the main experimental results, as described in supplemental material?
    \item[] Answer: \answerYes{} 
    \item[] Justification: Yes, we provide our code in the supplemental material.

\item {\bf Experimental setting/details}
    \item[] Question: Does the paper specify all the training and test details (e.g., data splits, hyperparameters, how they were chosen, type of optimizer, etc.) necessary to understand the results?
    \item[] Answer: \answerYes{} 
    \item[] Justification: Yes, all the training and test details are provided in Section~\ref{sec:experiments} and Appendix~\ref{appendix:Datasets}.

\item {\bf Experiment statistical significance}
    \item[] Question: Does the paper report error bars suitably and correctly defined or other appropriate information about the statistical significance of the experiments?
    \item[] Answer: \answerYes{} 
    \item[] Justification: Yes, our main experimental results in Table~\ref{main-llaga} and Table~\ref{main-gt} report error bars.

\item {\bf Experiments compute resources}
    \item[] Question: For each experiment, does the paper provide sufficient information on the computer resources (type of compute workers, memory, time of execution) needed to reproduce the experiments?
    \item[] Answer: \answerYes{} 
    \item[] Justification: Sufficient information on the computer resources is provided in Appendix~\ref{environment}.
    
\item {\bf Code of ethics}
    \item[] Question: Does the research conducted in the paper conform, in every respect, with the NeurIPS Code of Ethics \url{https://neurips.cc/public/EthicsGuidelines}?
    \item[] Answer: \answerYes{} 
    \item[] Justification: Yes, we have read the NeurIPS Code of Ethics and our paper conforms, in every respect, with the NeurIPS Code of Ethics.

\item {\bf Broader impacts}
    \item[] Question: Does the paper discuss both potential positive societal impacts and negative societal impacts of the work performed?
    \item[] Answer: \answerYes{} 
    \item[] Justification: We discuss both potential positive societal impacts and negative societal impacts of the work in Appendix~\ref{impacts}.
    
\item {\bf Safeguards}
    \item[] Question: Does the paper describe safeguards that have been put in place for responsible release of data or models that have a high risk for misuse (e.g., pretrained language models, image generators, or scraped datasets)?
    \item[] Answer: \answerNA{} 
    \item[] Justification: The paper poses no such risks. 

\item {\bf Licenses for existing assets}
    \item[] Question: Are the creators or original owners of assets (e.g., code, data, models), used in the paper, properly credited and are the license and terms of use explicitly mentioned and properly respected?
    \item[] Answer: \answerYes{} 
    \item[] Justification: We have cited the original paper that produced the code package or dataset. 

\item {\bf New assets}
    \item[] Question: Are new assets introduced in the paper well documented and is the documentation provided alongside the assets?
    \item[] Answer: \answerNA{} 
    \item[] Justification: The paper does not release new assets.

\item {\bf Crowdsourcing and research with human subjects}
    \item[] Question: For crowdsourcing experiments and research with human subjects, does the paper include the full text of instructions given to participants and screenshots, if applicable, as well as details about compensation (if any)? 
    \item[] Answer: \answerNA{} 
    \item[] Justification: The paper does not involve crowdsourcing nor research with human subjects.

\item {\bf Institutional review board (IRB) approvals or equivalent for research with human subjects}
    \item[] Question: Does the paper describe potential risks incurred by study participants, whether such risks were disclosed to the subjects, and whether Institutional Review Board (IRB) approvals (or an equivalent approval/review based on the requirements of your country or institution) were obtained?
    \item[] Answer: \answerNA{} 
    \item[] Justification: The paper does not involve crowdsourcing nor research with human subjects.

\item {\bf Declaration of LLM usage}
    \item[] Question: Does the paper describe the usage of LLMs if it is an important, original, or non-standard component of the core methods in this research? Note that if the LLM is used only for writing, editing, or formatting purposes and does not impact the core methodology, scientific rigorousness, or originality of the research, declaration is not required.
    \item[] Answer: \answerNA{} 
    \item[] Justification: The paper does not use LLMs as an important, original, or non-standard component. 

\end{enumerate}

\newpage
\appendix
\onecolumn

\section{Details of Datasets and Environment}
\label{appendix:Datasets}

\subsection{Datasets in LLaGA}
\label{appendix:datasets_llaga}

\begin{table}[h]
    \centering
    \caption{Dataset Statistics}
    \label{dataset_statistics_llaga}
    \small
    \begin{tabular}{ccccccc} 
        \toprule
        Datasets  & \# Nodes & \# Edges   & \# Classes & \# Features & \# homophily    & Split ratio      \\ 
        \midrule
        Cora      & 2,708    & 10,556       & 7    & 2,432      & 0.83  & 60\%/20\%/20\%        \\
        PubMed    & 19,717   & 88,648       & 3     & 2,432      & 0.79  & 60\%/20\%/20\%         \\
        Cornell    & 151   & 456        & 5    & 2,432       & 0.13  & 60\%/20\%/20\%       \\
        Texas    & 156   & 496       & 5      & 2,432     & 0.13   & 60\%/20\%/20\%       \\
        Wisconsin    & 244   & 846       & 5    & 2,432       & 0.18  & 60\%/20\%/20\%        \\
        Actor    & 9,228   & 272,862       & 5    & 2,432       & 0.67 & 60\%/20\%/20\%          \\
        \bottomrule
    \end{tabular}
\end{table}

Cornell, Texas, and Wisconsin: \href{https://www.cs.cmu.edu/afs/cs.cmu.edu/project/theo-20/www/data/}{https://www.cs.cmu.edu/afs/cs.cmu.edu/project/theo-20/www/data/}\\
For Actor, we get the raw texts of 5 classes from:
\begin{itemize}[leftmargin = 0.5cm]
    \item American\_film\_actors: \href{https://en.wikipedia.org/wiki/Category:American_film_actors}{https://en.wikipedia.org/wiki/Category:American\_film\_actors}
    \item American\_television\_actors: \href{https://en.wikipedia.org/wiki/Category:American_television_actors}{https://en.wikipedia.org/wiki/Category:American\_television\_actors}
    \item American\_screenwriters: \href{https://en.wikipedia.org/wiki/Category:American_screenwriters}{https://en.wikipedia.org/wiki/Category:American\_screenwriters}
    \item American\_stage\_actors: \href{https://en.wikipedia.org/wiki/Category:American_stage_actors}{https://en.wikipedia.org/wiki/Category:American\_stage\_actors}
    \item American\_film\_directors: \href{https://en.wikipedia.org/wiki/Category:American_film_directors}{https://en.wikipedia.org/wiki/Category:American\_film\_directors}
\end{itemize}

\subsection{Datasets in NAGphormer and VCR-Graphormer}
\label{appendix:datasets_gt}

\begin{table}[h]
    \centering
    \caption{Dataset Statistics}
    \label{dataset_statistics_gt}
    \small
    \begin{tabular}{ccccccc} 
        \toprule
        Datasets  & \# Nodes & \# Edges   & \# Classes & \# Features & \# homophily     & Split ratio     \\ 
        \midrule
        PubMed    & 19,717   & 88,651       & 3     & 500     & 0.79 & 60\%/20\%/20\%          \\
        Computers    &  13,752   & 441,512       & 10    & 767      & 0.70 & 60\%/20\%/20\%          \\
        Photo    & 7,650   & 223,538       & 8     & 745    & 0.77    & 60\%/20\%/20\%       \\
        Cornell    & 183   & 590        & 5  & 1,703        & 0.11    & 60\%/20\%/20\%     \\
        Texas    & 183   & 618       & 5      & 1,703     & 0.06   & 60\%/20\%/20\%       \\
        Wisconsin    & 251   & 998       & 5    & 1,703       & 0.16  & 60\%/20\%/20\%        \\
        Actor    & 7,600   & 33,544       & 5    & 931      & 0.24   & 60\%/20\%/20\%        \\
        \bottomrule
    \end{tabular}
\end{table}

All of these datasets can be accessed from the DGL library\footnote{\hyperlink{https://docs.dgl.ai/api/python/dgl.data.html}{https://docs.dgl.ai/api/python/dgl.data.html}}.

\subsection{Environment}
\label{environment}

The environment where our codes run is as follows: 
\begin{itemize}[leftmargin = 0.5cm]
    \item OS: Linux 5.4.0-131-generic
    \item CPU: Intel(R) Xeon(R) Gold 6348 CPU @ 2.60GHz
    \item GPU: GeForce RTX 3090
\end{itemize}

\section{More Experiments and Analysis on \MethodName}
\label{more_analysis}

\subsection{Attention Scores Allocated to Hops}
\label{analysis_attn_scores}
To further validate the effectiveness of \MethodName, we analyze the actual attention scores assigned to each hop by \MethodName~and HO, calculated as the weighted combination of scores from the gate module and the Transformer attention scores from the central node to each hop token. The results are summarized in Table~\ref{attn_scores}.
We can observe that: (1) On the homophilic graph, \MethodName~prioritizes near-hop neighbors, while HO prioritizes the central node. Therefore, \MethodName~refines this pattern by slightly up-weighting hops with marginally higher label-consistency.
(2) On the heterophilic graph, HO allocates disproportionately high attention to 1-hop neighbors, even when these neighbors have low label-consistency. In contrast, \MethodName~dynamically reallocates attention: it assigns higher scores to the central node and 2-hop neighbors with higher node-homophily.
These results confirm that the mechanisms of \MethodName~adapt to the type of graph, ensuring that attention aligns with the task.

\begin{table}[h]
    \centering
    \caption{Attention scores allocated to hops on homophilic and heterophilic graphs.}
    \label{attn_scores}
    \small
    \begin{tabular}{c|ccc|ccc} 
\toprule
\multirow{2}{*}{Templates} & \multicolumn{3}{c|}{LLaGA-Cora} & \multicolumn{3}{c}{LLaGA-Texas} \\ \cmidrule{2-7} 
                            & hop0      & hop1     & hop2     & hop0      & hop1     & hop2     \\ \midrule
HO                          & 0.4284    & 0.3069   & 0.2647   & 0.3721    & 0.2343   & 0.3936   \\ \midrule
\MethodName                        & 0.0323    & 0.6570   & 0.3108   & 0.3405    & 0.1441   & 0.5154   \\ \bottomrule
    \end{tabular}
    \vspace{-1.0em}
\end{table}

\subsection{Examples Demonstrating the Interpretability of \MethodName}
\label{examples}

To further illustrate why \MethodName~outperforms predefined token lists on heterophilic graphs, we present two representative examples from the Texas and WISCONSIN datasets with LLaGA (visualized in Figure~\ref{fig:examples}). \MethodName~reduces the attention of the 1-hop neighbor via the gate module and selects critical nodes via the selection module. For example, on the Texas dataset, the gate module makes the central node "18" pay more attention to 2-hop neighbors and itself. Moreover, the selection module increases the proportion of the feature of node "39" in the second hop token. This focuses on the aggregation on label-consistent nodes, correcting the prediction to label 2.

\begin{figure}[h]
\vspace{-0.5em}
  \centering
  \includegraphics[width=0.75\textwidth]{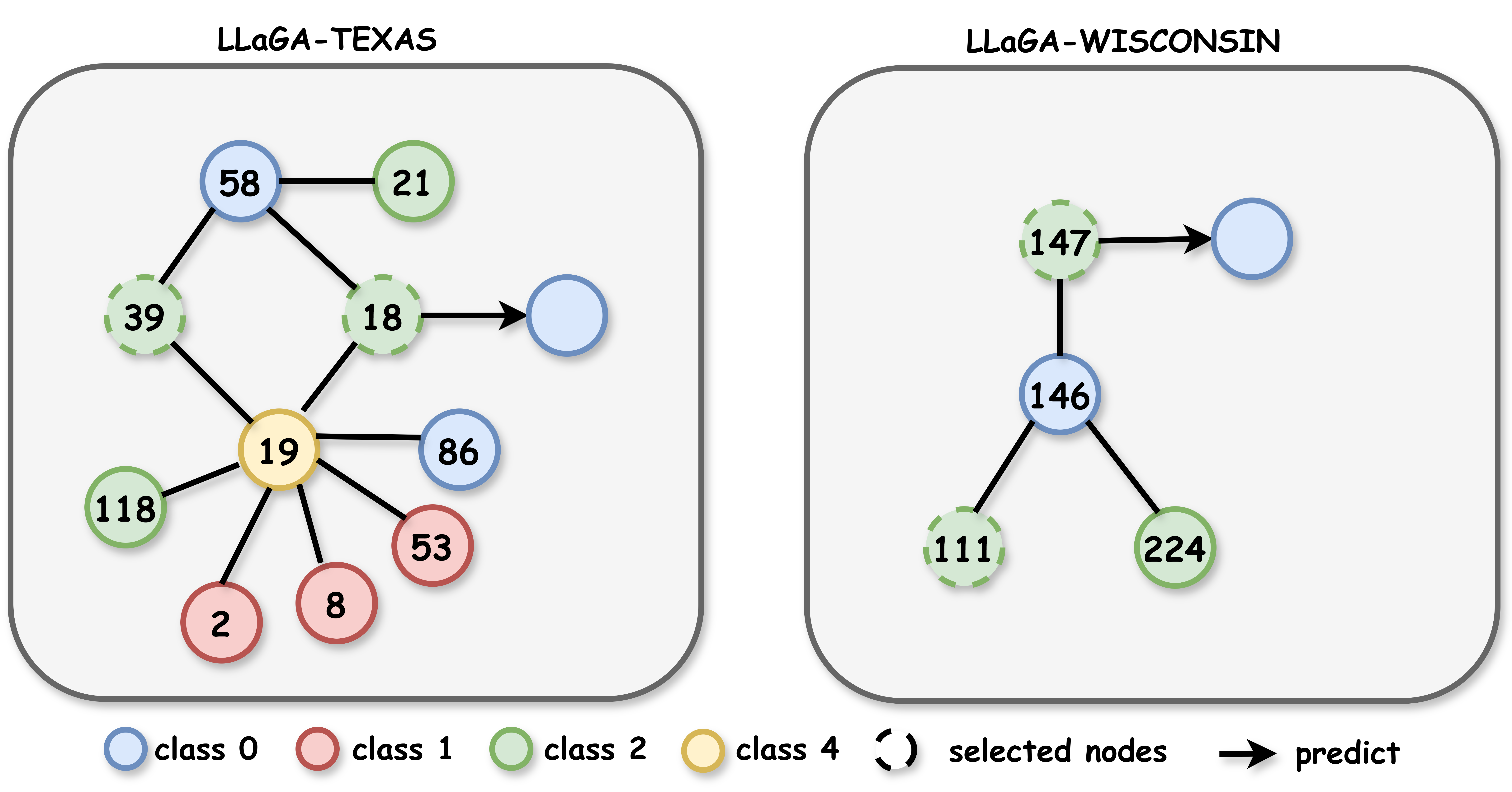}
    \caption{Examples demonstrating the interpretability of \MethodName.}
    \label{fig:examples}
    \vspace{-0.5em}
\end{figure}

\section{Limitations and Impacts}

\subsection{Limitations}
\label{limitations}
Here we provide discussions on limitations of our work. Currently, experiments of \MethodName~are primarily conducted on social and citation heterophilic graphs. 
In the future, we will extend \MethodName~to additional graph types (e.g., molecular interaction networks and knowledge graphs).
Furthermore, we will investigate the integration of \MethodName~with emerging graph learning paradigms (e.g., dynamic graph and heterogeneous graphs) to address evolving challenges in graph representation.

\subsection{Impacts}
\label{impacts}
\textbf{Positive Impacts}. Our method enhances Tokenized Graph Learning Models (TGLMs) by improving the adaptability of token lists, which can broadly benefit real-world applications. For example, TGLMs are critical for analyzing social networks, recommendation systems, and bioinformatics. By enabling TGLMs to better handle diverse graph structures, our work may improve the accuracy and efficiency of these applications, supporting data-driven decision-making in fields like public health, urban planning, and e-commerce.

\textbf{Negative Impacts}. As a foundational methodological contribution, our work focuses on general graph machine learning research and we do not foresee that our work shall have major direct negative societal impacts. 

\section{Propositions and Proofs of HO Template}\label{appendix:proof}
\subsection{Proof of Theorem~\ref{thm:MHOrecursive}}
\label{appendix:MHOrecursive}

\textbf{Base case k=0}:
By definition, $\mathbf{T}_{u,0}^{\text{HO}} = \mathbf{x}_u = \mathbf{H}_u^0$, so $\mathbf{M}_{0,0}^{\text{HO}} = 1$ and $\mathbf{M}_{0,i}^{\text{HO}} = \mathbf{0}$ for $i > 0$. 

\textbf{Base case k=1}:  
\begin{equation}
    \mathbf{T}_{u,1}^{\text{HO}} = \frac{1}{n} \sum_{v \in \mathcal{N}_u^1} \mathbf{T}_{v,0}^{\text{HO}} = \frac{1}{n} \sum_{v \in \mathcal{N}_u^1} \mathbf{x}_v = \frac{1}{n} \mathbf{H}_u^1.
\end{equation}

Thus, \( \mathbf{T}_{u,1}^{\text{HO}} = \mathbf{M}_{1,1}^{\text{HO}} \mathbf{H}_u^1 \), where \( \mathbf{M}_{1,1}^{\text{HO}} = \frac{1}{n} \). For \( i=0 \), \( \mathbf{M}_{1,0}^{\text{HO}} = \mathbf{0} \). This matches Rule 2 (\( \mathbf{M}_{1,0}^{\text{HO}} = \mathbf{M}_{0,1}^{\text{HO}} = \mathbf{0} \)) and Rule 4 (\( \mathbf{M}_{1,1}^{\text{HO}} = \frac{1}{n} \left( \mathbf{M}_{0,0}^{\text{HO}} + (n-1)\mathbf{M}_{0,2}^{\text{HO}} \right) = \frac{1}{n} \)).

\textbf{Inductive step k=t}:  
Assume \( \mathbf{T}_{u,t-1}^{\text{HO}} = \sum_{i=0}^{t-1} \mathbf{M}_{t-1,i}^{\text{HO}} \mathbf{H}_u^i \) holds. For \( \mathbf{T}_{u,t}^{\text{HO}} \):
\begin{equation}
\mathbf{T}_{u,t}^{\text{HO}} = \frac{1}{n} \sum_{v \in \mathcal{N}_u^1} \mathbf{T}_{v,t-1}^{\text{HO}} = \frac{1}{n} \sum_{v \in \mathcal{N}_u^1} \sum_{i=0}^{t-1} \mathbf{M}_{t-1,i}^{\text{HO}} \mathbf{H}_v^i.
\end{equation}
By definition, \( \sum_{v \in \mathcal{N}_u^1} \mathbf{H}_v^i = \mathbf{H}_u^{i+1} + (n-1)\mathbf{H}_u^{i-1} \). Substituting this into the equation:  
\begin{equation}
\mathbf{T}_{u,t}^{\text{HO}} = \frac{1}{n} \left[ \mathbf{M}_{t-1,0}^{\text{HO}} \mathbf{H}_u^1 + \sum_{i=1}^{t-1} \mathbf{M}_{t-1,i}^{\text{HO}} \left( \mathbf{H}_u^{i+1} + (n-1)\mathbf{H}_u^{i-1} \right) \right].
\end{equation}
Rearranging terms by \( \mathbf{H}_u^i \):  
\begin{equation}
\mathbf{T}_{u,t}^{\text{HO}} = \mathbf{M}_{t-1,1}^{\text{HO}} \mathbf{H}_u^0 + \sum_{i=1}^t \frac{1}{n} \left( \mathbf{M}_{t-1,i-1}^{\text{HO}} + (n-1)\mathbf{M}_{t-1,i+1}^{\text{HO}} \right) \mathbf{H}_u^i.
\end{equation}
This matches the recursive rules for $\mathbf{M}_{t,0}^{\text{HO}}$ and $\mathbf{M}_{t,i}^{\text{HO}}$ (Rules 2, 3, and 4). Thus, the induction holds. 

\subsection{Proposition1: Contributions Relate to the Parity of Hop}
\label{appendix:MHOprop1}

\begin{proposition}[Contributions Relate to the Parity of Hop]\label{prop:MHOprop1}
 For any $k \geq 0$, $\mathbf{T}_{u,k}^{\text{HO}}$ aggregates features exclusively from neighbors with hop counts of the same parity as $k$. 
Formally:
\begin{itemize}[leftmargin = 0.5cm]
    \item For even $k=2m$: $\mathbf{M}_{2m, 2n+1}^{\text{HO}} = \mathbf{0}$, $n \geq 0$, and $\mathbf{M}_{2m, 2n}^{\text{HO}} \neq \mathbf{0}$, $0 \leq n \leq m$;
    \item For odd $k=2m+1$: $\mathbf{M}_{2m+1, 2n}^{\text{HO}} = \mathbf{0}$, $n \geq 0$, and $\mathbf{M}_{2m+1, 2n+1}^{\text{HO}} \neq \mathbf{0}$, $0 \leq n \leq m$.
\end{itemize}
\end{proposition}

\textbf{Base Case k=0 and k=1}:
\begin{itemize}[leftmargin = 0.5cm]
    \item For $k=0$ (even), $\mathbf{T}_{u,0}^{\text{HO}} = \mathbf{x}_u$, so $\mathbf{M}_{0,0}^{\text{HO}} = 1$ and $\mathbf{M}_{0,i}^{\text{HO}} = \mathbf{0}$ for $i > 0$.
    \item For $k=1$ (odd), $\mathbf{T}_{u,1}^{\text{HO}} = \frac{1}{n} \sum_{v \in \mathcal{N}_u^1} \mathbf{x}_v$, so $\mathbf{M}_{1,1}^{\text{HO}} = \frac{1}{n}$ and $\mathbf{M}_{1,i}^{\text{HO}} = \mathbf{0}$ for $i \neq 1$. 
\end{itemize}

\textbf{Inductive Step k=t:}  
  Assume the property holds for $k=t$. Consider $k=t+1$:
\begin{itemize}[leftmargin = 0.5cm]
    \item If $t$ is even ($t=2m$), then $t+1=2m+1$ (odd). By Rule 2 of Theorem~\ref{thm:MHOrecursive}, $\mathbf{M}_{2m+1,0}^{\text{HO}} = \mathbf{M}_{2m,1}^{\text{HO}} = \mathbf{0}$. For $i \geq 1$, $\mathbf{M}_{2m+1,i}^{\text{HO}} = \frac{1}{n} \left( \mathbf{M}_{2m,i-1}^{\text{HO}} + (n-1)\mathbf{M}_{2m,i+1}^{\text{HO}} \right)$. Since $\mathbf{M}_{2m,i-1}^{\text{HO}}$ and $\mathbf{M}_{2m,i+1}^{\text{HO}}$ are non-zero only if $i-1$ and $i+1$ are even (i.e., $i$ is odd), $\mathbf{M}_{2m+1,i}^{\text{HO}}$ is non-zero only when $i$ is odd.
    \item If $t$ is odd ($t=2m+1$), then $t+1=2m+2$ (even). By Rule 2 of Theorem~\ref{thm:MHOrecursive}, $\mathbf{M}_{2m+2,0}^{\text{HO}} = \mathbf{M}_{2m+1,1}^{\text{HO}} \neq \mathbf{0}$. For $i \geq 1$, $\mathbf{M}_{2m+2,i}^{\text{HO}} = \frac{1}{n} \left( \mathbf{M}_{2m+1,i-1}^{\text{HO}} + (n-1)\mathbf{M}_{2m+1,i+1}^{\text{HO}} \right)$. Since $\mathbf{M}_{2m+1,i-1}^{\text{HO}}$ and $\mathbf{M}_{2m+1,i+1}^{\text{HO}}$ are non-zero only if $i-1$ and $i+1$ are odd (i.e., $i$ is even), $\mathbf{M}_{2m+2,i}^{\text{HO}}$ is non-zero only when $i$ is even. 
\end{itemize}

Thus, the property holds for all $k \geq 0$.

\subsection{Proposition2: Monotonic Decay of Row Contributions}
\label{appendix:MHOprop2}

\begin{proposition}[Monotonic Decay of Row Contributions]
\label{prop:MHOprop2}
Within each row of $\mathbf{M}^{\text{HO}}$, nonzero contributions monotonically decay as the hop increases. Formally:
\begin{itemize}[leftmargin = 0.5cm]
    \item For even $k=2m$: $\mathbf{M}_{2m, 2i}^{\text{HO}} > \mathbf{M}_{2m, 2(i+1)}^{\text{HO}}$, $0 \leq i \leq m-1$;
    \item For odd $k=2m+1$: $\mathbf{M}_{2m+1, 2i+1}^{\text{HO}} > \mathbf{M}_{2m+1, 2(i+1)+1}^{\text{HO}}$, $0 \leq i \leq m-1$.
\end{itemize}
\end{proposition}

\textbf{Base Case m=0}:
\begin{itemize}[leftmargin = 0.5cm]
\item For $k=0$ (even), the row has only $\mathbf{M}_{0,0}^{\text{HO}} = 1$.  
\item For $k=1$ (odd), the row has only $\mathbf{M}_{1,1}^{\text{HO}} = \frac{1}{n}$.  
\end{itemize}
\textbf{Inductive Step m=t}: Assume the property holds for $m=t$. Consider $m=t+1$:  
\begin{itemize}[leftmargin = 0.5cm]
\item For even $k=2(t+1)$:  
    By Rule 2 and Rule 4, $\mathbf{M}_{2(t+1),0}^{\text{HO}} = \mathbf{M}_{2t+1,1}^{\text{HO}}$. For $i \geq 1$,  
    \begin{equation}
        \mathbf{M}_{2(t+1),2i}^{\text{HO}} = \frac{1}{n} \left( \mathbf{M}_{2t+1,2i-1}^{\text{HO}} + (n-1)\mathbf{M}_{2t+1,2i+1}^{\text{HO}} \right).
    \end{equation}

    By the inductive hypothesis, $\mathbf{M}_{2t+1,2i-1}^{\text{HO}} < \mathbf{M}_{2t+1,2i-3}^{\text{HO}}$ and $\mathbf{M}_{2t+1,2i+1}^{\text{HO}} < \mathbf{M}_{2t+1,2i-1}^{\text{HO}}$, so  
    \begin{equation}
    \mathbf{M}_{2(t+1),2i}^{\text{HO}} < \frac{1}{n} \left( \mathbf{M}_{2t+1,2i-3}^{\text{HO}} + (n-1)\mathbf{M}_{2t+1,2i-1}^{\text{HO}} \right) = \mathbf{M}_{2(t+1),2(i-1)}^{\text{HO}}.
    \end{equation}
\item For odd $k=2(t+1)+1$:  
    By Rule 4,  
    \begin{equation}
    \mathbf{M}_{2(t+1)+1,2i+1}^{\text{HO}} = \frac{1}{n} \left( \mathbf{M}_{2(t+1),2i}^{\text{HO}} + (n-1)\mathbf{M}_{2(t+1),2(i+1)}^{\text{HO}} \right).
    \end{equation}
    By the inductive hypothesis, $\mathbf{M}_{2(t+1),2i}^{\text{HO}} < \mathbf{M}_{2(t+1),2(i-1)}^{\text{HO}}$ and $\mathbf{M}_{2(t+1),2(i+1)}^{\text{HO}} < \mathbf{M}_{2(t+1),2i}^{\text{HO}}$, so  
    \begin{equation}
    \mathbf{M}_{2(t+1)+1,2i+1}^{\text{HO}} < \mathbf{M}_{2(t+1)+1,2i-1}^{\text{HO}}.
    \end{equation}
\end{itemize}
Thus, the property holds for all $m \geq 0$.

\subsection{Proposition3: Monotonic Decay of Column Contributions}
\label{appendix:MHOprop3}

\begin{proposition}[Monotonic Decay of Column Contributions]\label{prop:MHOprop3} 
Within each column of $\mathbf{M}^{\text{HO}}$, nonzero contributions monotonically decay as the token depth increases . Formally:
\begin{itemize}[leftmargin = 0.5cm]
    \item For even $i=2j$: $\mathbf{M}_{2m,2j}^{\text{HO}} > \mathbf{M}_{2(m+1),2j}^{\text{HO}}$ for $0 \leq j \leq m$;
    \item For odd $i=2j+1$: $\mathbf{M}_{2m+1,2j+1}^{\text{HO}} > \mathbf{M}_{2(m+1)+1,2j+1}^{\text{HO}}$ for $0 \leq j \leq m$.
\end{itemize}
\end{proposition}

\textbf{Base Case m=0}:
\begin{itemize}[leftmargin = 0.5cm]
\item For $i=0$ (even), $\mathbf{M}_{0,0}^{\text{HO}} = 1 > \mathbf{M}_{2,0}^{\text{HO}} = \mathbf{M}_{1,1}^{\text{HO}} = \frac{1}{n}$.  
\item For $i=1$ (odd), $\mathbf{M}_{1,1}^{\text{HO}} = \frac{1}{n} > \mathbf{M}_{3,1}^{\text{HO}} = \frac{1}{n} \left( \mathbf{M}_{2,0}^{\text{HO}} + (n-1)\mathbf{M}_{2,2}^{\text{HO}} \right) = \frac{2n-1}{n^3}$.  
\end{itemize}

\textbf{Inductive Step m=t}: Assume the property holds for $m=t$. Consider $m=t+1$:  
\begin{itemize}[leftmargin = 0.5cm]
\item For even $i=2j$: By Rule 4, 
\begin{equation}
\mathbf{M}_{2(t+1),2j}^{\text{HO}} = \frac{1}{n} \left( \mathbf{M}_{2t+1,2j-1}^{\text{HO}} + (n-1)\mathbf{M}_{2t+1,2j+1}^{\text{HO}} \right)
\end{equation}

By the inductive hypothesis, $\mathbf{M}_{2t+1,2j-1}^{\text{HO}} > \mathbf{M}_{2t+3,2j-1}^{\text{HO}}$ and $\mathbf{M}_{2t+1,2j+1}^{\text{HO}} > \mathbf{M}_{2t+3,2j+1}^{\text{HO}}$, so $\mathbf{M}_{2(t+1),2j}^{\text{HO}} > \mathbf{M}_{2(t+2),2j}^{\text{HO}}$.  
\item For odd $i=2j+1$: By Rule 4, 
\begin{equation}
    \mathbf{M}_{2(t+1)+1,2j+1}^{\text{HO}} = \frac{1}{n} \left( \mathbf{M}_{2(t+1),2j}^{\text{HO}} + (n-1)\mathbf{M}_{2(t+1),2(j+1)}^{\text{HO}} \right).
\end{equation}
 By the inductive hypothesis, $\mathbf{M}_{2(t+1),2j}^{\text{HO}} > \mathbf{M}_{2(t+2),2j}^{\text{HO}}$ and $\mathbf{M}_{2(t+1),2(j+1)}^{\text{HO}} > \mathbf{M}_{2(t+2),2(j+1)}^{\text{HO}}$, so $\mathbf{M}_{2(t+1)+1,2j+1}^{\text{HO}} > \mathbf{M}_{2(t+2)+1,2j+1}^{\text{HO}}$.  
\end{itemize}
Thus, the property holds for all $m \geq 0$. Then we have 
\begin{equation}
\frac{\mathbf{M}_{k,k-2}^{\text{HO}}}{\mathbf{M}_{k,k}^{\text{HO}}} = (k-1)n - (k-2) > (k-2)(n-1).
\end{equation}
Therefore, even within $\mathbf{T}_{u,k}^{\text{HO}}$, the far-hop neighbors are exponentially less influential than the near-hop ones.

\subsection{Proof of Theorem~\ref{thm:MHOAttnrecursive}}
\label{appendix:MHOAttnthm}
The Transformer's attention mechanism aggregates value vectors $\mathbf{v}_u^i = \mathbf{T}_{u,i}^{\text{HO}} \mathbf{W}_V$ (with $\mathbf{W}_V$ as the value projection matrix) using the attention scores $\alpha_i$:  
\begin{equation}
\begin{aligned}
    \sum_{i=0}^L \alpha_i \mathbf{v}_u^i = \sum_{i=0}^L \alpha_i \mathbf{T}_{u,i}^{\text{HO}} \mathbf{W}_V.
\end{aligned}
\end{equation}
Substituting $\mathbf{T}_{u,i}^{\text{HO}} = \sum_{j=0}^i \mathbf{M}_{i,j}^{\text{HO}} \mathbf{H}_u^j$ (from~\ref{eq:M_HO_define}) and $\mathbf{H}_u^j = \sum_{v \in \mathcal{N}_u^j} \mathbf{x}_v$ (sum of $j$-hop neighbor features), we get:  
\begin{equation}
    \begin{aligned}
&\sum_{i=0}^L \alpha_i \mathbf{v}_u^i \\
=& \sum_{i=0}^L \alpha_i  \mathbf{T}_{u,i}^{\text{HO}} \mathbf{W}_V \\
=& \sum_{i=0}^L \alpha_i  \sum_{j=0}^i ( \mathbf{M}_{i,j}^{\text{HO}} \mathbf{H}_u^j ) \mathbf{W}_V \\
=& \sum_{j=0}^L \left( \sum_{i=j}^L \alpha_i \mathbf{M}_{i,j}^{\text{HO}} \right) \left( \sum_{v \in \mathcal{N}_u^j} \mathbf{x}_v \right) \mathbf{W}_V. 
    \end{aligned}
\end{equation} 
By linearity of summation, the effective attention allocated to a specific $k$-hop neighbor $v \in \mathcal{N}_u^k$ is the coefficient of $\mathbf{x}_v \mathbf{W}_V$ in this expression. By Property 1 (parity restriction) of $\mathbf{M}_{i,k}^{\text{HO}}$, $\mathbf{M}_{i,k}^{\text{HO}} = \mathbf{0}$ when $i$ and $k$ have different parity. Thus, only terms with $i \equiv k \, \text{mod} \, 2$ contribute:  
\begin{equation}
\hat{\alpha}_k = \sum_{i=k}^L \alpha_i \mathbf{M}_{i,k}^{\text{HO}} = \sum_{i=k, \, i \equiv k \, \text{mod} \, 2}^L \alpha_i \mathbf{M}_{i,k}^{\text{HO}}.
\end{equation} 

For property 1, given $k_1 < k_2$ with $k_1 \equiv k_2 \, \text{mod} \, 2$, consider the effective attention $\hat{\alpha}_{k_1}$ and $\hat{\alpha}_{k_2}$. By Property 2 (row monotonicity) of $\mathbf{M}_{i,k}^{\text{HO}}$, $\mathbf{M}_{i,k_1}^{\text{HO}} > \mathbf{M}_{i,k_2}^{\text{HO}}$ for all $i \geq k_2$. Thus:  
\begin{equation}
    \hat{\alpha}_{k_1} = \sum_{i=k_1, \, i \equiv k_1 \, \text{mod} \, 2}^L \alpha_i \mathbf{M}_{i,k_1}^{\text{HO}} > \sum_{i=k_2, \, i \equiv k_2 \, \text{mod} \, 2}^L \alpha_i \mathbf{M}_{i,k_1}^{\text{HO}} > \sum_{i=k_2, \, i \equiv k_2 \, \text{mod} \, 2}^L \alpha_i \mathbf{M}_{i,k_2}^{\text{HO}} = \hat{\alpha}_{k_2}.
\end{equation}

And for property 2, from the derivation of $\hat{\alpha}_k$, the effective attention allocated to a $k$-hop neighbor $v$ is the coefficient of $\mathbf{x}_v \mathbf{W}_V$ in the aggregated value vector:  
\begin{equation}
    \sum_{i=0}^L \alpha_i \mathbf{V}_u^i = \sum_{v \in \mathcal{V}} \hat{\alpha}_v \mathbf{x}_v \mathbf{W}_V.
\end{equation}

For $v \in \mathcal{N}_u^k$, $\hat{\alpha}_v$ is determined by the sum of $\alpha_i \mathbf{M}_{i,k}^{\text{HO}}$ over $i \equiv k \, \text{mod} \, 2$. Since $\mathbf{M}_{i,k}^{\text{HO}}$ and $\alpha_i$ are global to the token list (not dependent on individual nodes $v$), all $k$-hop neighbors $v_1, v_2 \in \mathcal{N}_u^k$ share the same $\hat{\alpha}_k$. Thus, $\hat{\alpha}_{v_1} = \hat{\alpha}_{v_2} = \hat{\alpha}_k$.

\subsection{Proof of Theorem~\ref{thm:MHOBound}}
\label{appendix:MHOBound}

We assume the existence of a linear classifier, parameterized by $W_C$, which satisfies the condition $\mathbf{H}^0 \mathbf{W}_C = \mathbf{Y}$. We can express $\mathbf{H}^0 = \mathbf{Y} \mathbf{W}_C^{-1}$.

Using the triangle inequality for Frobenius norms:
\begin{equation}
    \left\| \mathbf{H}_u^0 - \hat{\mathbf{A}} \mathbf{H}^0 \right\|_F = \left\| \mathbf{H}_u^0 - \sum_{i=0}^L \hat{\alpha}_i \sum_{v \in \mathcal{N}_u^i} \mathbf{H}_v^0 \right\|_F \leq \sum_{i=0}^L \hat{\alpha}_i \sum_{v \in \mathcal{N}_u^i} \left\| \mathbf{H}_u^0 - \mathbf{H}_v^0 \right\|_F.
\end{equation}
Assume raw features are Lipschitz continuous with respect to labels: $\left\| \mathbf{H}_u^0 - \mathbf{H}_v^0 \right\|_F \leq L \left\| \mathbf{y}_u - \mathbf{y}_v \right\|_F$ for the constant $L$. For one-hot labels, $\left\| \mathbf{y}_u - \mathbf{y}_v \right\|_F = \sqrt{2}$ if $\mathbf{y}_u \neq \mathbf{y}_v$, and $0$ otherwise. Thus:  
\begin{equation}
    \sum_{v \in \mathcal{N}_u^i} \left\| \mathbf{H}_u^0 - \mathbf{H}_v^0 \right\|_F \leq L \cdot \sqrt{2} \cdot |\{ v \in \mathcal{N}_u^i \mid \mathbf{y}_u \neq \mathbf{y}_v \}|.
\end{equation}
The number of $i$-hop neighbors with different labels is $|\mathcal{N}_u^i| (1 - C_u^i)$, where $C_u^i = \frac{|\{ v \in \mathcal{N}_u^i \mid \mathbf{y}_u = \mathbf{y}_v \}|}{|\mathcal{N}_u^i|}$ is the label consistency ratio. Substituting this into the inequality:  
\begin{equation}
\sum_{v \in \mathcal{N}_u^i} \left\| \mathbf{H}_u^0 - \mathbf{H}_v^0 \right\|_F \leq \sqrt{2} L \cdot |\mathcal{N}_u^i| (1 - C_u^i).
\end{equation} 
Finally, substituting back into the smoothness bound:  
\begin{equation}
\| \mathbf{H}_u^0 - \hat{\mathbf{A}} \mathbf{H}^0 \|_F \leq \sqrt{2} L \sum_{i=0}^L \hat{\alpha}_i |\mathcal{N}_u^i| (1 - C_u^i).
\end{equation}  

\section{Details about ND}
\label{appendix:proof_of_ND}
\subsection{Definition and Recursive Properties of ND}
Consider ND for the central node $u$, constructed via a fixed-shape computational tree of depth $L$ with neighbor sample size $n$ per hop. Let $L' = \sum_{i=0}^L n^i$ denote the token list length. Define $\mathbf{M}_{i,j}^{\text{ND}}$ as the number of times $j$-hop neighbors of $u$ appear in the $i$-th layer of the token list (where layers are indexed by hop distance). The coefficients $\mathbf{M}_{i,j}^{\text{ND}}$ satisfy:
\begin{enumerate}[leftmargin=0.5cm]
    \item $\mathbf{M}^{\text{ND}}_{0,0} = 1$ (root node, 0-hop). 
    \item $\mathbf{M}^{\text{ND}}_{1,0} = 0$, $\mathbf{M}^{\text{ND}}_{1,1} = 1$ (layer 1 contains only 1-hop neighbors).
    \item $\mathbf{M}^{\text{ND}}_{k,0} = n\mathbf{M}^{\text{ND}}_{k-1,1}$.  
    \item For $i \geq 1$ and $j \geq 1$: $\mathbf{M}^{\text{ND}}_{i,j} = \mathbf{M}^{\text{ND}}_{i-1,j-1} + (n-1)\mathbf{M}^{\text{ND}}_{i-1,j+1}$. 
\end{enumerate}

\begin{proof}
For $i=0$ (the root layer), ND contains only $u$ itself, so $\mathbf{M}^{\text{ND}}_{0,0} = 1$ and $\mathbf{M}^{\text{ND}}_{0,j} = 0$ for $j > 0$. Rule 1 holds. \\
For $i=1$, since $u$ does not appear in layer 1, $\mathbf{M}^{\text{ND}}_{1,0} = 0$; $\mathcal{N}_{u}^{1}$ appear exactly once per sampled node, so $\mathbf{M}^{\text{ND}}_{1,1} = 1$. Rule 2 holds. \\  
For 0-hop in general Layers ($i \geq 1$), the count of $\mathcal{N}_{u}^{0}$ in layer $k$ depends on the count of $\mathcal{N}_{u}^{1}$ in layer $k-1$. Each of the node in $\mathcal{N}_{u}^{1}$ in layer $k-1$ generates $n$ children in layer $k$, but $u$ itself appears as a parent of these children. Therefore, $\mathbf{M}^{\text{ND}}_{k,0} = n\mathbf{M}^{\text{ND}}_{k-1,1}$. Rule 3 holds. \\
For other hops in general Layers ($i \geq 1$), A $j$-th hop neighbor in layer $i$ can be derived from two sources: (1) A $(j-1)$-th hop neighbor in layer $i-1$ (parent node, contributing 1 occurrence), (2) $(j+1)$-th hop neighbors in layer $i-1$ (siblings of the parent node, contributing $(n-1)$ occurrences due to sampling). Therefore, $\mathbf{M}^{\text{ND}}_{i,j} = \mathbf{M}^{\text{ND}}_{i-1,j-1} + (n-1)\mathbf{M}^{\text{ND}}_{i-1,j+1}$. Rule 4 holds.
\end{proof}

\subsection{Monotonicity Properties of ND}
To characterize how ND distributes node occurrences across layers, we analyze the monotonicity of $\mathbf{M}^{\text{ND}}_{i,j}$. Specifically, we observe two critical monotonicity properties:
\begin{itemize}[leftmargin = 0.5cm]
    \item Within-Layer Decay: $\mathbf{M}^{\text{ND}}_{i,j} > \mathbf{M}^{\text{ND}}_{i,j+2}$ for $\mathbf{M}^{\text{ND}}_{i,j} \neq 0$ (closer hops appear more frequently within a layer).
    \item Cross-Layer Growth: $\mathbf{M}^{\text{ND}}_{i,j} < \mathbf{M}^{\text{ND}}_{i+2,j}$ for $\mathbf{M}^{\text{ND}}_{i,j} \neq 0$ (deeper layers amplify the count of fixed-hop neighbors).  
\end{itemize}

\begin{proof} 
By induction, assume $\mathbf{M}^{\text{ND}}_{i,j} > \mathbf{M}^{\text{ND}}_{i,j+2}$ holds for $i = t$ and any $j$. For $i = t+1$ and $j > 0$:  
\begin{equation}
    \mathbf{M}^{\text{ND}}_{t+1,j} = \mathbf{M}^{\text{ND}}_{t,j-1} + (n-1)\mathbf{M}^{\text{ND}}_{t,j+1} > \mathbf{M}^{\text{ND}}_{t,j+1} + (n-1)\mathbf{M}^{\text{ND}}_{t,j+3} = \mathbf{M}^{\text{ND}}_{t+1,j+2}.
\end{equation}
This follows from the inductive hypothesis $\mathbf{M}^{\text{ND}}_{t,j-1} > \mathbf{M}^{\text{ND}}_{t,j+1}$ and $\mathbf{M}^{\text{ND}}_{t,j+1} > \mathbf{M}^{\text{ND}}_{t,j+3}$.

And for $i = t+1$ and $j = 0$:
\begin{equation}
    \mathbf{M}^{\text{ND}}_{t+1,0} = n\mathbf{M}^{\text{ND}}_{t,1} > \mathbf{M}^{\text{ND}}_{t,1} + (n-1)\mathbf{M}^{\text{ND}}_{t,3} = \mathbf{M}^{\text{ND}}_{t+1,2}.
\end{equation}
Therefore, Within-Layer Decay holds.

For Cross-Layer Growth, assume $\mathbf{M}^{\text{ND}}_{t,j} < \mathbf{M}^{\text{ND}}_{t+2,j}$ for $i=t$ and any $j$. For $i=t+2$:  
\begin{equation}
\mathbf{M}^{\text{ND}}_{t+2,j} = \mathbf{M}^{\text{ND}}_{t+1,j-1} + (n-1)\mathbf{M}^{\text{ND}}_{t+1,j+1} > \mathbf{M}^{\text{ND}}_{t-1,j-1} + (n-1)\mathbf{M}^{\text{ND}}_{t-1,j+1} = \mathbf{M}^{\text{ND}}_{t,j}.
\end{equation}

Therefore, Cross-Layer Growth holds.\end{proof}

\subsection{Effective Attention Allocation of ND}
Let $\alpha \in \mathbb{R}^{1 \times L'}$ denotes the attention scores assigned to the ND token list $\mathbf{T}^{\text{ND}}_{u,1}, \mathbf{T}^{\text{ND}}_{u,2}, \dots, \mathbf{T}^{\text{ND}}_{u,L'}$ (normalized such that $\sum_{i=1}^{L'} \alpha_i = 1$). For a $k$-hop neighbor $v \in \mathcal{N}_u^k$ of $u$, the effective attention allocated to $v$ is:  
\begin{equation}
    \beta_{u,v} = \phi_{L,k} \cdot \alpha_{u,v},
\end{equation}
where $\phi_{L,k} = \sum_{i=k, i \equiv k \, \text{mod} \, 2}^L \mathbf{M}^{\text{ND}}_{i,k}$ is the weight of $k$-hop neighbors, and $\alpha_{u,v}$ is the direct attention score between $u$ and $v$.  
\begin{proof}
The aggregated value vector of ND is:  
\begin{equation}
\begin{aligned}
    & \sum_{i=1}^{L'} \alpha_i \mathbf{V}_i \\
    =& \sum_{i=0}^{L}\sum_{k=0}^{i}\mathbf{M}^{\text{ND}}_{i,k}\sum_{v\in \mathcal{N}_u^k}\alpha_{u,v}V_{v} \\
    =& \sum_{k=0}^{L}(\sum_{i=k}^{L}\mathbf{M}^{\text{ND}}_{i,k})\sum_{v\in \mathcal{N}_u^k}\alpha_{u,v}V_{v} \\
    =& \sum_{k=0}^{L}\sum_{v\in \mathcal{N}_u^k}((\sum_{i=k, i \equiv k \, \text{mod} \, 2}^{L}\mathbf{M}^{\text{ND}}_{i,k})\alpha_{u,v})\mathbf{x}_{v}W_V,
\end{aligned}
\end{equation} 
where $\mathbf{V}_j$ is the value vector of the $j$-th $k$-hop neighbor. By definition, $\phi_{L,k} = \sum_{i=k}^L \mathbf{M}^{\text{ND}}_{i,k}$, so the effective attention $\beta_{u,v} = \phi_{L,k} \cdot \alpha_{u,v}$.  
\end{proof}

\subsection{Hop-Priority Bias in ND}
This allocation exhibits two critical properties:
\begin{enumerate}[leftmargin = 0.5cm]
    \item Near-Hop Dominance: $\phi_{L,k} > \phi_{L,k+2}$ for all $k \leq L$ (closer hops have higher total weights).
    \item Layer Parity Bias:
    \begin{itemize}[leftmargin = 0.5cm]
        \item If $L$ is odd ($L=2k+1$), odd hops ($k=1,3,\dots$) have $\phi_{L,k} > \phi_{L,k-1}$ and $\phi_{L,k} > (n-1)\phi_{L,k+1}$.  
        \item If $L$ is even ($L=2k$), even hops ($k=0,2,\dots$) have $\phi_{L,k} > \phi_{L,k-1}$ and $\phi_{L,k} > (n-1)\phi_{L,k+1}$.  
    \end{itemize}
\end{enumerate}
\begin{proof}
From within-layer decay, $\mathbf{M}^{\text{ND}}_{i,k} > \mathbf{M}^{\text{ND}}_{i,k+2}$ for all $i \geq k$. Therefore:  
\begin{equation}
    \phi_{L,k} = \sum_{i=k, i \equiv k \, \text{mod} \, 2}^L \mathbf{M}^{\text{ND}}_{i,k} > \sum_{i=k, i \equiv k \, \text{mod} \, 2}^L \mathbf{M}^{\text{ND}}_{i,k+2} = \phi_{L,k+2}.
\end{equation}
Near-Hop Dominance holds.

For $L=2k+1$ (odd), consider $\phi_{L,2t+1}$:  
\begin{equation}
\begin{aligned}
    & \phi_{L,2t+1} \\
    =& \sum_{i=2t+1, i \equiv 2t+1 \, \text{mod} \, 2}^{2k+1} \mathbf{M}^{\text{ND}}_{i,2t+1} \\
    =& \sum_{i=2t, i \equiv 2t \, \text{mod} \, 2}^{2k} \mathbf{M}^{\text{ND}}_{i,2t} + (n-1) \sum_{i=2t+2, i \equiv 2t+2 \, \text{mod} \, 2}^{2k} \mathbf{M}^{\text{ND}}_{i,2t+2} \\
    =& \phi_{L,2t} + (n-1)\phi_{L,2t+2}.
\end{aligned}
\end{equation}
Therefore, $\phi_{L,2t+1} > \phi_{L,2t}$ and $\phi_{L,2t+1} > (n-1)\phi_{L,2t+2}$.  

For $L=2k$ (even), consider $\phi_{L,2t+2}$:  
\begin{equation}
\begin{aligned}
& \phi_{L,2t+2} \\
=& \sum_{i=2t+2, i \equiv 2t+2 \, \text{mod} \, 2}^{2k} \mathbf{M}^{\text{ND}}_{i,2t+2} \\
=& \sum_{i=2t+1, i \equiv 2t+1 \, \text{mod} \, 2}^{2k-1} \mathbf{M}^{\text{ND}}_{i,2t+1} + (n-1) \sum_{i=2t+3, i \equiv 2t+3 \, \text{mod} \, 2}^{2k-1} \mathbf{M}^{\text{ND}}_{i,2t+3} \\
=& \phi_{L,2t+1} + (n-1)\phi_{L,2t+3}.
\end{aligned}
\end{equation}
Therefore, $\phi_{L,2t+2} > \phi_{L,2t+1}$ and $\phi_{L,2t+2} > (n-1)\phi_{L,2t+3}$.  Layer Parity Bias holds.
\end{proof}

\subsection{Smoothness Bound of ND}
Let $\hat{\mathbf{A}} \in \mathbb{R}^{1 \times N}$ the attention vector for all nodes derived from ND. The smoothness of node $u$'s representation satisfies:  
\begin{equation}
    \| \mathbf{H}_u^0 - \hat{\mathbf{A}} \mathbf{H}^0 \|_F \leq \sqrt{2}L\frac{1}{1+\frac{1}{\frac{\sum_{i=0}^{L}\phi_{L,i}|\mathcal{N}^{i}_{u}|}{\sum_{i=0}^{L}\phi_{L,i}|\mathcal{N}^{i}_{u}|C_{u}^{i}}-1}\frac{\eta_{u}}{\gamma_{u}}},
\end{equation}
where $\gamma_u = \mathbb{E}_{v \in \mathcal{N}_u^i, \mathbf{y}_u = \mathbf{y}_v} \exp\left(\frac{\mathbf{q}_u \mathbf{k}_v^\top}{\sqrt{h}}\right)$, and $\eta_u = \mathbb{E}_{v \in \mathcal{N}_u^i, \mathbf{y}_u \neq \mathbf{y}_v} \exp\left(\frac{\mathbf{q}_u \mathbf{k}_v^\top}{\sqrt{h}}\right)$.
\begin{proof}
The smoothness metric of ND is:  
\begin{equation}
    \| \mathbf{H}_u^0 - \hat{\mathbf{A}} \mathbf{H}^0 \|_F = \left\| \mathbf{H}_u^0 - \sum_{i=1}^{L'} \alpha_i \mathbf{T}^{\text{ND}}_{u,i} \right\|_F,
\end{equation}
where $\alpha_i$ are the attention scores assigned to the tokens in ND.  

By Theorem 9, each $k$-hop neighbor $v$ of $u$ appears $\phi_{L,k}$ times in the token list, with effective attention $\beta_{k,v} = \phi_{L,k} \cdot \alpha_{u,v}$. Thus, the aggregated token list can be rewritten as:
\begin{equation}
\sum_{i=1}^{L'} \alpha_i \mathbf{T}^{\text{ND}}_{u,i} = \sum_{k=0}^L \sum_{v \in \mathcal{N}_u^k} \beta_{k,v} \mathbf{H}_v^0.
\end{equation}
 
Using the Lipschitz assumption $\| \mathbf{H}_u^0 - \mathbf{H}_v^0 \|_F \leq L \| \mathbf{y}_u - \mathbf{y}_v \|_F$ and the one-hot label property $\| \mathbf{y}_u - \mathbf{y}_v \|_F = \sqrt{2}$ for $\mathbf{y}_u \neq \mathbf{y}_v$, we bound the smoothness metric:  
\begin{equation}
\| \mathbf{H}_u^0 - \hat{\mathbf{A}} \mathbf{H}^0 \|_F \leq \sqrt{2} L \sum_{k=0}^L \sum_{v \in \mathcal{N}_u^k, \mathbf{y}_u \neq \mathbf{y}_v} \beta_{k,v}.
\end{equation}
Let $C_u^k = \frac{|\{ v \in \mathcal{N}_u^k \mid \mathbf{y}_u = \mathbf{y}_v \}|}{|\mathcal{N}_u^k|}$ denote the label consistency of $k$-hop neighbors. The number of $k$-hop neighbors with different labels is $|\mathcal{N}_u^k| (1 - C_u^k)$. Substituting $\beta_{k,v} = \phi_{L,k} \cdot \alpha_{u,v}$, we get:  
\begin{equation}
\sum_{v \in \mathcal{N}_u^k, \mathbf{y}_u \neq \mathbf{y}_v} \beta_{k,v} = \phi_{L,k} \sum_{v \in \mathcal{N}_u^k, \mathbf{y}_u \neq \mathbf{y}_v}. \alpha_{u,v}.
\end{equation}

The attention scores $\alpha_{u,v}$ are softmax-normalized:  
\begin{equation}
\alpha_{u,v} = \frac{\exp\left(\frac{\mathbf{q}_u \mathbf{k}_v^\top}{\sqrt{h}}\right)}{\sum_{v' \in \mathbf{T}} \exp\left(\frac{\mathbf{q}_u \mathbf{k}_{v'}^\top}{\sqrt{h}}\right)}.
\end{equation}

The sum over $\alpha_{u,v}$ for differing labels becomes:  
\begin{equation}
\sum_{v \in \mathcal{N}_u^k, \mathbf{y}_u \neq \mathbf{y}_v} \alpha_{u,v} = \frac{|\mathcal{N}_u^k| (1 - C_u^k) \eta_u}{\sum_{i=0}^{L}\phi_{L,i}(|\mathcal{N}_u^i| C_u^i \gamma_u + |\mathcal{N}_u^i| (1 - C_u^i) \eta_u)}.
\end{equation}

Therefore, we bound the smoothness metric:  
\begin{equation}
\begin{aligned}
    &\| \mathbf{H}_u^0 - \hat{\mathbf{A}} \mathbf{H}^0 \|_F \\
    \leq& \sqrt{2} L \sum_{k=0}^L \sum_{v \in \mathcal{N}_u^k, \mathbf{y}_u \neq \mathbf{y}_v} \beta_{k,v}\\
    =& \sqrt{2} L \sum_{k=0}^L\frac{\phi_{L,k}|\mathcal{N}_u^k| (1 - C_u^k) \eta_u}{\sum_{i=0}^{L}\phi_{L,i}(|\mathcal{N}_u^i| C_u^i \gamma_u + |\mathcal{N}_u^i| (1 - C_u^i) \eta_u)}\\
    =&\sqrt{2}L\frac{1}{1+\frac{\sum_{i=1}^{L}\phi_{L,i}|\mathcal{N}_u^i|C_{u}^{i}}{\sum_{i=1}^{L}\phi_{L,i}|\mathcal{N}_u^i|(1-C_{u}^{i})}\frac{\eta_{u}}{\gamma_{u}}}\\
    =&\sqrt{2}L\frac{1}{1+\frac{1}{\frac{\sum_{i=1}^{L}\phi_{L,i}|\mathcal{N}_u^i|(1-C_{u}^{i})}{\sum_{i=1}^{L}\phi_{L,i}|\mathcal{N}_u^i|C_{u}^{i}}}\frac{\eta_{u}}{\gamma_{u}}}\\
    =&\sqrt{2}L\frac{1}{1+\frac{1}{\frac{\sum_{i=1}^{L}\phi_{L,i}|\mathcal{N}_u^i|}{\sum_{i=1}^{L}\phi_{L,i}|\mathcal{N}_u^i|C_{u}^{i}}-1}\frac{\eta_{u}}{\gamma_{u}}}.
\end{aligned}
\end{equation}
\end{proof}

On homophilic graphs, where $C_u^i$ is uniformly high, the smoothness bound remains tight, enabling strong model performance. However, on heterophilic graphs, the structural bias of ND becomes critical: while $C_u^i$ increases with odd hop, the attention weight $\phi_{L,i}$ decreases with hop distance. This mismatch, where the model amplifies attention to near hops (with low $C_u^i$) and suppresses far hops (with high $C_u^i$), weakens the smoothness bound, limiting the model's ability to learn meaningful representations. These theoretical findings align with the results of preliminary experiments (Table~\ref{empirical_observations}), where ND underperforms on heterophilic datasets due to this rigid hop-priority bias.

\section{Theoretical Analysis of \MethodName}
\label{method_theoretical}

\subsection{Adaptive Attention Aggregation with Gate Module}
\label{adaptive_attention_aggregation}
The attention given to the node \( v \) in the \( i \)-th hop neighbors of node $u$ is:  
\begin{equation}
    \hat{\beta}_{u,i,v} = \hat{\alpha}_{u,i} \cdot \beta_{u,i,v} =\frac{\alpha_{u,i} \cdot \hat{\mathbf{s}}_{u,i}}{\langle \alpha_u \cdot \mathbf{s}_{u}\rangle} \cdot \beta_{u,i,v}.
\end{equation}
\begin{proof}
The aggregated value vector after adjustment is:  
\begin{equation}
\sum_{i=0}^L \hat{\alpha}_{u,i} \mathbf{V}_u^i = \sum_{i=0}^L \hat{\alpha}_{u,i} \mathbf{T}_u^i \mathbf{W}_V.
\end{equation}
Substituting \( \mathbf{T}_u^i = \sum_{v \in \mathcal{G}_u^i} \beta_{u,i,v} \mathbf{H}_v^0 \) and \( \hat{\alpha}_{u,i} \), we get:  
\begin{equation}
\sum_{i=0}^L \frac{\alpha_{u,i} \cdot \hat{\mathbf{s}}_{u,i}}{\langle \alpha_u \cdot \mathbf{s}_{u}\rangle} \sum_{v \in \mathcal{G}_u^i} \beta_{u,i,v} \mathbf{H}_v^0 \mathbf{W}_V = \sum_{i=0}^L \sum_{v \in \mathcal{G}_u^i} \left( \frac{\alpha_{u,i} \cdot \hat{\mathbf{s}}_{u,i}}{\langle \alpha_u \cdot \mathbf{s}_{u}\rangle} \cdot \beta_{u,i,v} \right) \mathbf{H}_v^0 \mathbf{W}_V.
\end{equation}
This holds.
\end{proof}


\subsection{Relationship with ND}
\label{relationship_with_ND}
ND scales the attention of the hops by $\phi_{L,k} = \sum_{i=k, i \equiv k \, \text{mod} \, 2}^L \mathbf{M}^{\text{ND}}_{i,k}$, where \( \mathbf{M}^{\text{ND}}_{i,k} \) is the hop matrix of ND. Therefore, \( \hat{\beta}_{u,k,v}^{\text{ND}} = \phi_{L,k} \cdot \beta_{u, k, v}^{\text{ND}} \). 
For \MethodName, setting $\mathcal{G}_u^i=\mathcal{N}_u^i$, $\beta_{u,k,v}=\beta_{u, k, v}^{\text{ND}}$, \( \hat{\alpha}_{u,k} = \phi_{L,k} \) recovers the attention pattern of ND, i.e., 
\begin{equation}
s_{u,k}\propto \frac{\phi_{L,k}}{\alpha_{u,i}}
\end{equation}
Therefore, ND is also a special case of \MethodName.

\subsection{Smooth Bound of \MethodName}
\label{smooth_bound_of_our_method}
For simplicity of our analysis, assume that $\hat{\mathbf{s}}_{u,i}$ is the scores allocated to the $i$-th hop for the central node $u$, and $\sum_{i=0}^{L}\hat{\mathbf{s}}_{u,i}=L+1$, the Frobenius norm \( \| \mathbf{H}_u^0 - \hat{\mathbf{A}} \mathbf{H}^0 \|_F \) is bounded by:  
\begin{equation}
\| \mathbf{H}_u^0 - \hat{\mathbf{A}} \mathbf{H}^0 \|_F \leq \sqrt{2}L\frac{\sum_{i=0}^{L}\hat{\mathbf{s}}_{u,i}|\mathcal{G}_u^i|(1-C_{u}^{i})}{(\sum_{i=0}^{L}|\mathcal{G}_u^i|(1-C_{u}^{i}))+(\sum_{i=0}^{L}|\mathcal{G}_u^i|C_{u}^{i})\frac{\eta_{u}}{\gamma_{u}}}
\end{equation} 
\begin{proof}
\begin{equation}
\begin{aligned}
    &\| \mathbf{H}_u^0 - \hat{\mathbf{A}} \mathbf{H}^0 \|_F\\
    \leq& \sqrt{2}L\sum_{i=0}^{L}\hat{\mathbf{s}}_{u,i} \sum_{v\in \mathcal{G}_u^i,y_v\ne y_u}\beta_{u,i,v}\\
    =&\sqrt{2}L\frac{(\sum_{i=0}^{L}\hat{\mathbf{s}}_{u,i}(\sum_{v\in \mathcal{G}_u^i,y_{v}\ne y_{u}}1))\gamma_{u}}{(\sum_{i=0}^{L}(\sum_{v\in \mathcal{G}_u^i,y_{v}\ne y_{u}}1))\gamma_{u}+(\sum_{i=0}^{L}(\sum_{v\in \mathcal{G}_u^i,y_{v}=y_{u}}1))\eta_{u}}\\
    =&\sqrt{2}L\frac{(\sum_{i=0}^{L}\hat{\mathbf{s}}_{u,i}|\mathcal{G}_u^i|(1-C_{u}^{i}))\gamma_{u}}{(\sum_{i=0}^{L}|\mathcal{G}_u^i|(1-C_{u}^{i}))\gamma_{u}+(\sum_{i=0}^{L}|\mathcal{G}_u^i|C_{u}^{i})\eta_{u}}\\
    =&\sqrt{2}L\frac{\sum_{i=0}^{L}\hat{\mathbf{s}}_{u,i}|\mathcal{G}_u^i|(1-C_{u}^{i})}{(\sum_{i=0}^{L}|\mathcal{G}_u^i|(1-C_{u}^{i}))+(\sum_{i=0}^{L}|\mathcal{G}_u^i|C_{u}^{i})\frac{\eta_{u}}{\gamma_{u}}}.
\end{aligned}
\end{equation}
\end{proof}
\textit{Analysis:}
If we directly aggregate all of the nodes from the $i$-th hop and allocate the same score for all hops, the bound is the same as graph transformers~\citep{xing2024less}. Now we fix $|\mathcal{G}_u^i|$, we should assign high gate scores $\hat{\mathbf{s}}_{u}$ to hops with less label-inconsistent nodes, i.e., $|\mathcal{G}_u^i|(1-C_{u}^{i})$, to reduce the value of $\| \mathbf{H}_u^0 - \hat{\mathbf{A}} \mathbf{H}^0 \|_F$. However, we can not know which hop has a higher value of $C_u^i$. So, in general, $|\mathcal{G}_u^i|$ is the same for each hop. Let $|\mathcal{G}_u^i|=|G|$, and the smooth bouid is
\begin{equation}
    \sqrt{2}L\frac{\sum_{i=0}^{L}\hat{\mathbf{s}}_{u,i}(1-C_{u}^{i})}{\sum_{i=0}^{L}(1-C_{u}^{i})+\frac{\eta_{u}}{\gamma_{u}}\sum_{i=0}^{L}C_{u}^{i}}.
\end{equation}
So in order to reduce the value of $\| \mathbf{H}_u^0 - \hat{\mathbf{A}} \mathbf{H}^0 \|_F$, the higher $\hat{\mathbf{s}}_{u,i}$ will be assigned to the hop with the higher value of $C_{u}^{i}$, theoretically explaining the effectiveness of the gate module. Furthermore, if we allocate all scores to the hop with the highest $C_{u}^{i}$ (define as $C_{u}^{I}$), then we have:
\begin{equation}
    \| \mathbf{H}_u^0 - \hat{\mathbf{A}} \mathbf{H}^0 \|_F \leq \mathcal{J}(1-C_{u}^{I}),
\end{equation}
where $\mathcal{J}$ is a constant value, while the smooth bound of the predefined token lists and graph transformers is related to other task-irrelevant hops.

\subsection{Special Cases: Frozen LLM and Hybrid Token Lists}
\subsubsection{Frozen LLM Adaptation}
When using a frozen LLM, the attention mechanism cannot be fine-tuned. \MethodName~approximates the adjustment by scaling tokens directly:  
\begin{equation}
    {\mathbf{T}}_u^0 = \mathbf{s}_{u,0} \mathbf{H}_u^0, \quad {\mathbf{T}}_u^i = \mathbf{s}_{u,i}\sum_{v \in \mathcal{G}_u^i} \beta_{u,i,v} \mathbf{H}_v^0.
\end{equation}
The aggregated value vector becomes:  
\begin{equation}
\begin{aligned}
    & \sum_{i=0}^L {\alpha}_{u,i} V_u^i \\
    =& \sum_{i=0}^L {\alpha}_{u,i} {\mathbf{T}}_u^i \mathbf{W}_V \\
    =& \sum_{i=0}^L \sum_{v \in \mathcal{G}_u^i} ({\alpha}_{u,i} \mathbf{s}_{u,i}\beta_{u,i,v}) \mathbf{H}_v^0 \mathbf{W}_V\\
    =& \sum_{i=0}^L ({\alpha}_{u,i} \mathbf{s}_{u,i}) \sum_{v \in \mathcal{G}_u^i} \beta_{u,i,v} \mathbf{H}_v^0 \mathbf{W}_V,
\end{aligned}
\end{equation}
equivalent to applying \( s_{u,i} \) as a multiplicative weight to each hop token.

\subsubsection{Compatibility with Extended Token Lists (e.g., VCR-Graphormer)}
A critical strength of \MethodName~is its plug-and-play compatibility with models that include additional tokens (e.g., cluster-based tokens in VCR-Graphormer). We formalize this compatibility below. Consider a model like VCR-Graphormer, where the token list includes \( C \) cluster-based tokens. Let \( \mathbf{C}^j \) denote the \( j \)-th cluster, and \( p_{u,v} \) the Personalized PageRank score of node \( v \) relative to \( u \). The extended token list \( \mathbf{T}_u' \) is:  
\begin{equation}
\begin{split}
&\ \ \ \ \mathbf{T}_u^0 = \mathbf{s}_{u,0} \cdot \mathbf{H}_u^0\\
&\quad \mathbf{T}_u^i = \mathbf{s}_{u,i} \cdot \sum_{v \in \mathcal{G}_u^i} \beta_{u,i,v} \mathbf{H}_v^0 \ (i=1,\dots,L)\\
&\quad \mathbf{T}_u^{L+j} = \sum_{v \in \mathbf{C}^j} p_{u,v} \mathbf{H}_v^0 \ (j=1,\dots,C).
\end{split}
\end{equation}

The model's aggregated value vector using \( \mathbf{T}_u \) is:  
\begin{equation}
\sum_{i=0}^{L+C} \dot{\alpha}_i \cdot \mathbf{V}_u^i = \sum_{i=0}^{L+C} \dot{\alpha}_i \cdot \mathbf{T}_u^i \mathbf{W}_V,
\end{equation}
where \( \mathbf{W}_V \) is the value projection matrix. Substituting \( \mathbf{T}_u^i \):  
\begin{equation}
= \sum_{i=0}^L \dot{\alpha}_i \cdot \mathbf{s}_{u,i} \cdot \left( \sum_{v \in \mathcal{G}_u^i} \beta_{u,i,v} \mathbf{H}_v^0 \right) \mathbf{W}_V + \sum_{j=1}^C \dot{\alpha}_{L+j} \cdot \left( \sum_{v \in \mathbf{C}^j} p_{u,v} \mathbf{H}_v^0 \right) \mathbf{W}_V.
\end{equation}
The effective attention score for any node \( v \) is:  
\begin{equation}
\hat{\alpha}_v = \sum_{i=0}^L \left[ v \in \mathcal{G}_u^i \right] \dot{\alpha}_i \cdot \mathbf{s}_{u,i} \cdot \beta_{u,i,v} + \sum_{j=1}^C \left[ v \in \mathbf{C}^j \right] \dot{\alpha}_{L+j} \cdot p_{u,v},
\end{equation} 
where \( [\cdot] \) is an indicator function. When cluster-based attention scores are zero (\( \dot{\alpha}_{L+j} = 0 \)), the model reduces to the base \MethodName, confirming its compatibility as a drop-in plugin.

\end{document}